\documentclass{article} 
\usepackage[margin=1.25in]{geometry}

\author{Matus Telgarsky\thanks{Department of Computer Science and Engineering,
University of California, San Diego.
Email:
\texttt{<mtelgars@cs.ucsd.edu>}.}}

\usepackage[round,comma]{natbib}
\usepackage{amsmath}
\usepackage{amssymb,graphicx,dsfont}

\usepackage{amsthm}
\usepackage{thmtools}
\usepackage{subfig} 
\usepackage[usenames,dvipsnames]{color}
\usepackage{hyperref}
\usepackage{cleveref}
\usepackage{mathrsfs}
\usepackage{latexsym} 

\usepackage{tikz}
\usetikzlibrary{calc,intersections,through,positioning,matrix}
\usetikzlibrary{decorations.markings}

\def\1{\mathds 1}
\def\R{\mathbb R}
\def\Z{\mathbb Z}

\def\bfe{\mathbf e}

\def\bG{\mathbb G}

\def\cD{\mathcal D}
\def\cF{\mathcal F}

\def\cH{\mathcal H}

\def\cK{\mathcal K}
\def\cL{\mathcal L}

\def\cO{\mathcal O}
\def\cP{\mathcal P}
\def\cR{\mathcal R}
\def\cS{\mathcal S}

\def\cX{\mathcal X}
\def\cY{\mathcal Y}

\def\scrC{\mathscr C}

\def\scrP{\mathscr P}

\def\fF{\mathfrak F}

\def\SPAN{\textup{span}}

\def\dom{\textup{dom}}

\newcommand{\ip}[2]{\left\langle #1, #2 \right \rangle}
\newcommand{\argmin}{\operatornamewithlimits{arg\,min}}

\newcommand{\LL}[2]{\lambda_{#2}^{(#1)}}

\newcommand{\red}[1]{{\color{red} #1}}
\newcommand{\blue}[1]{{\color{blue} #1}}

\numberwithin{equation}{section}
\declaretheorem[numberlike=equation]{theorem}
\declaretheorem[numberlike=theorem]{lemma}
\declaretheorem[numberlike=theorem]{proposition}
\declaretheorem[numberlike=theorem]{corollary}
\declaretheoremstyle[%
qed={\ensuremath\Diamond}
]{remstyle}
\declaretheorem[numberlike=theorem,style=remstyle]{definition}

\title{Statistical Consistency of Finite-dimensional Unregularized Linear Classification}
\date{}

\begin{document}
\maketitle

\begin{abstract}
    This manuscript studies statistical properties of linear classifiers obtained
    through minimization of an unregularized convex risk over
    a finite sample. Although the results are explicitly finite-dimensional,
    inputs may be passed through feature maps; in this way, in addition to
    treating the consistency of logistic regression, this analysis also handles
    boosting over a finite weak learning class with, for instance, the
    exponential, logistic, and hinge losses.
    In this finite-dimensional setting, it is still possible to fit arbitrary
    decision boundaries: scaling the complexity of the weak learning
    class with the sample size leads to the optimal classification risk almost surely.
\end{abstract}

\section{Introduction}

Binary linear classification operates as follows: obtain a new instance, determine
a set of real-valued features, form their weighted combination, and output a label which
is positive iff this combination is nonnegative.  The interpretability, empirical
performance, and theoretical depth of this scheme have all contributed to its continued
popularity~\citep{freund_schapire_adaboost,friedman_hastie_tibshirani_statboost,
caruana_empirical}.

In order to obtain the coefficients in the above weighting, convex optimization is
typically employed.  Specifically, rather than just trying to pick the weighting which
makes the fewest mistakes over a finite sample --- which is computationally intractable
--- consider instead paying attention to the \emph{amount} by which these combinations
clear the zero threshold, a quantity called the \emph{margin}.  Applying a convex
penalty to these margins yields a convex optimization procedure, specifically one
which can be specialized into
both
logistic regression and AdaBoost.


Statistical analyses of this scheme predominately follow two paths.
The first path is a parameter estimation approach; 
positive and negative instances are interpreted as drawn from a
family of distributions, indexed by the combination weights above, and the
convex scheme is performing a maximum
likelihood search for these parameters~\citep{friedman_hastie_tibshirani_statboost}.
This provides one
way to analyze logistic regression, specifically the ability of the above convex
optimization to recover these parameters; these analyses of course require such
parameters to exist, and usually for the full problem to obey certain
regularity conditions~\citep{lebanon_mle,gm_logistic}.

The second approach is focused on the case of binary classification, with an
interpretation of the data generation process taking a background role.
Indeed, in this setting,
optimal parameters may simply fail to exist~\citep{schapire_adaboost_convergence_rate},
and the convex optimization procedure can produce unboundedly large weightings.
Analyses first focused on the separable case, showing that AdaBoost
approximately maximizes \emph{normalized} margins, and that this leads to good
generalization~\citep[Chapter 5 and the references therein]{schapire_freund_book}.
It is historically interesting that this setting, which entails the non-existence
of the best parameters, is diametrically opposed to the parameter estimation
setting above.

In order to produce a more general analysis, it was necessary to control the unbounded
iterates.  This has been achieved either implicitly through regularization
\citep{blv_regularized_boosting}, or explicitly with an early stopping rule
\citep{bartlett_traskin_adaboost,zhang_yu_boosting,schapire_freund_book}.
Those analyses which handle the case of AdaBoost
(cf. the work of \citet{bartlett_traskin_adaboost}
and \citet[Chapter 12]{schapire_freund_book}), are sensitive both to the
choice of exponential loss, to the choice of minimization scheme, and to the choice of
stopping condition.

The goal of this manuscript is to analyze the setting of minimizing an unregularized
convex loss applied to a finite sample (i.e., just like logistic regression and AdaBoost),
but for a large class of loss functions, and without any demands on the optimization
algorithm beyond an ability to attain arbitrarily small error.

\subsection{Contribution}
In more detail, the primary characteristics of the presented analysis are as follows.
\begin{description}
    \item[Any minimization scheme.]  The oracle producing approximate solutions
        to the convex problem can output iterates which have any norm; they must simply
        be close in objective value to the optimum.  The intent of this choice is
        twofold: for practitioners, it means that focusing on minimizing this objective
        value suffices; for theorists, it means that the wild deviations caused by
        these unbounded norms are not actually an issue.

    \item[Many convex losses.]
        The analysis applies to any convex loss which is positive at the origin,
        and zero in the limit.  (Some results also require differentiability at
        the origin.)
        In particular, the analysis handles the popular choice
        of using the logistic loss,
        but also applies to the exponential and hinge losses.
        (For a discussion on the difficulties of generalizing from the
        exponential loss, please see the work of
        \citet[Section 4]{bartlett_traskin_adaboost}.)
\end{description}

The main limitation of the presented analysis is that the set of features, or
\emph{weak learners}, must be finite.
This weakness can be circumvented in the setting of boosting, where the complexity of
the feature set can increase with the availability of data; it will be shown that
the popular choice of decision trees fit this regime nicely.

\subsection{Outline}

A summary of the manuscript, and its organization, are as follows.  Briefly,
primary notation and technical background appear in \Cref{sec:setup}.

\Cref{sec:impossibility} presents an impossibility result, which forces the structure
of subsequent content.  Specifically, with no bound on the iterates, it is
in general impossible to control the deviations between the \emph{empirical convex
risk} (the convex surrogate risk over the observed finite sample),
and the \emph{true convex risk}
(the convex surrogate risk over the source distribution).

The solution is to break the input space into two pieces:
a \emph{hard core}, where there exists an imperfect yet optimal parameter vector,
and the hard core's complement, where it is possible to have zero mistakes, albeit
giving up on the existence of a minimizer to the true convex risk.  This material
appears in \Cref{sec:hard_cores}.

The hard core has direct entailments on the structure of the convex risk.
Specifically, \Cref{sec:hard_core:true_risk} establishes first that
the true risk has quantifiable curvature over the hard core, and effectively zero
error over the rest of the space.  Additionally, with high probability,
this structure carries over to any sampled instance.

The significance of first proving properties of the true risk, and then carrying them
over to the sample, is that quantities dictating the structure of the empirical
convex risk are \emph{sample independent}.  Consequently, finite sample guarantees,
which appear in \Cref{sec:deviations}, display a number of terms which are properties
of the true convex risk, and not simply opaque random variables derived from the
sample.  It is thus possible to control many such bounds together; the eventual
consistency results, appearing in \Cref{sec:consistency}, simply combine the finite
sample guarantees, which all share the same primary structural quantities, together
with standard probability techniques.  As discussed previously, in order to fit
arbitrary decision boundaries, structural risk minimization is employed, and it
is furthermore established that decision trees with a constraint on the location of
splits meet the requisite structural risk minimization condition.


Note that all proofs, as well as some supporting technical material,
appear in a variety of appendices.

\section{Notation}
\label{sec:setup}
\begin{definition}
    Instances $x\in\cX$ will have associated labels $y\in\cY = \{-1,+1\}$.
    $\mu$ will always denote a probability measure over $\cX\times \cY$,
    with only occasional mention of the related $\sigma$-algebra.
\end{definition}

To achieve generality sufficient to treat boosting, instances will not be worked with
directly, but instead through a family of feature maps, or weak learners.

\begin{definition}
    Let $\cH = \{h_i\}_{i=1}^n$
    denote a finite set of (measurable) functions $\cH\ni h : \cX \to [-1,+1]$.
    Call a pair $(\cH,\mu)$ a \emph{linear classification problem}.
    For convenience, let $H$ denote a (bounded) linear operator with elements of $\cH$ as
    abstract columns: given any weighting $\lambda\in\R^n$,
    \[
        H\lambda = \sum_{i=1}^n \lambda_i h_i.
    \]
    For convenience, define related classes of functions
    \begin{align*}
        \SPAN(\cH,b) &:= \left\{
            H\lambda : \lambda\in\R^n, \|\lambda\|_1 \leq b
        \right\},
        \\
        \SPAN(\cH) &:= \bigcup_{b=1}^\infty\SPAN(\cH, b) = \left\{
            H\lambda : \lambda\in\R^n
        \right\}.
        \qedhere
    \end{align*}
\end{definition}

The class $\SPAN(\cH)$ will be the search space for linear classification; if
for instance $\cH$ consists of projection maps, then this is the standard setting of
linear regression, however in general it can be viewed as a boosting problem.  That the
range of the function family is fixed specifically to $[-1,+1]$ is irrelevant,
however compactness of this output space is used throughout.

\begin{definition}
    $\Phi$ contains all convex losses $\phi$ which are positive at the origin, and
    satisfy $\lim_{z\to-\infty} \phi(z) = 0$.
\end{definition}

This manuscript makes the choice of writing losses as nondecreasing functions; in this
notation, three examples are the exponential loss $\exp(z)$, logistic loss
$\ln(1+\exp(z))$, and hinge loss $\max\{0, 1+z\}$.  Some of the consistency results
will also require the loss to be differentiable at the origin; this requirement, which
is satisfied by the three preceding examples, will be explicitly stated.

\begin{definition}
    Given a probability measure $\mu$, a loss $\phi\in\Phi$, a function class $\cF$,
    and
    arbitrary element $f\in \cF$, the corresponding
    risk functional, and optimal risk, are
    \[
        \cR_\phi(f) := \int \phi(-yf(x)) d\mu(x,y),
        \qquad
        \qquad
        \cR_\phi(\cF) = \inf_{f\in \cF} \cR_\phi(f).
    \]
    When a sample $\cS := \{(x_i,y_i)\}_{i=1}^m$ is provided, let $\cR^m_\phi$ denote
    the corresponding empirical risk, meaning the convex risk corresponding to
    the empirical measure $\mu_m(C) := m^{-1} \sum_{i=1}^m \1((x_i,y_i)\in C)$,
    thus $\cR_\phi^m(f) = m^{-1} \sum_i \phi(-y_i f(x_i))$.
    Lastly, let $\cL$ denote the classification risk $\cL(y,y') := \1(y \neq y')$,
    and overload the notation for risks so that
    \[
        \cR_\cL(f) := \int \cL(y, 2\cdot\1(f(x)\geq 0)-1) d\mu(x,y),
        \qquad
        \qquad
        \cR_\cL(\cF) = \inf_{f\in \cF} \cR_\cL(f).
        \qedhere
    \]
\end{definition}
Typically, some function class $\cH$, a particular weighting $\lambda\in\R^n$,
and perhaps a sample of size $m$ will be available,
and example relevant risks are $\cR_\phi(H\lambda)$, $\cR_\cL^m(H\lambda)$,
$\cR_\phi(\SPAN(\cH))$.

\begin{definition}
    The requirement placed on the minimization oracle is that, for any $\cH$,
    $\phi\in\Phi$,
    finite sample of size $m$, and \emph{suboptimality} $\rho>0$, the oracle can produce
    $\lambda\in\R^n$ with $\cR_\phi^m(H\lambda) \leq \cR_\phi^m(\SPAN(\cH)) + \rho$.
\end{definition}

The theorems themselves will avoid any reliance on this oracle, and
their guarantees will
hold with any $\rho$-suboptimal $\lambda$ as input;
this manuscript is concerned with statistical properties of these predictors.
However, note briefly that for many losses of interest,
in particular the hinge, logistic, and exponential losses,
oracles satisfying the above guarantee exist.

\begin{proposition}[\citep{nesterov_book,primal_dual_boosting}]
    \label{fact:opt_oracles_exist}
    Let a linear classification problem $(\cH,\mu)$, finite sample of size $m$,
    and suboptimality $\rho>0$ be given.  Suppose:
    \begin{enumerate}
        \item Either $\phi$ is Lipschitz continuous, attains its infimum, and
            subgradient descent is employed;
        \item Or $\phi$ is in the convex cone generated by the logistic and
            exponential losses, and coordinate descent is employed (as in AdaBoost);
    \end{enumerate}
    then $\textup{poly}(1/\rho)$
    iterations suffice to produce a $\rho$-suboptimal iterate $\lambda\in\R^n$.
\end{proposition}
(The proof, in \Cref{sec:app:setup}, is mostly a reduction to known results
regarding subgradient and coordinate descent.)

Lastly, this manuscript adopts a form of event-defining notation common in probability
theory.

\begin{definition}
    Given a function $f : A \to B$ and binary relation $\sim$, define
    $[f \sim b] := \{a \in A : f(a) \sim b\}$; for example
    $[f > 0] := \{a \in A: f(a) > 0\} = f^{-1}((0,\infty))$.  At times, the variables
    will also be provided, for instance $[bf(a) > 0] = \{(a,b) \subseteq A\times B:
    bf(a) > 0\}$.
\end{definition}


\section{An impossibility result}
\label{sec:impossibility}
The stated goal of allowing iterates to have unbounded norms is at odds
with the task of bounding the convex risk $\cR_\phi$.

\begin{proposition}
    \label{fact:separable_case_impossibility}
    There exists a linear classification problem
    $(\cH,\mu)$ with the following characteristics.
    \begin{enumerate}
        \item
            $\cX$ is the square $[-1,+1]^2$, and $\cH$ consists of the two projection
            maps.
        \item $\mu$ has countable support.
        \item There exists a perfect separator, albeit with zero margin.
        \item For any $\phi\in\Phi$, 
            $\cR_\phi(\SPAN(\cH)) = 0$.
        \item Let any finite sample $\{(x_i,y_i)\}_{i=1}^m$, any $b>0$,
            and any $\phi\in\Phi$ 
            be given.
            Then there exists a maximum margin solution $\hat\lambda$,
            i.e., a solution satisfying
            \[
                \argmin_{i\in [m]} \frac{y_i (H\hat\lambda)(x_i)}{\|\hat\lambda\|_1}
                = \sup\left\{
                    \argmin_{i\in [m]}
                    y_i (H\lambda)(x_i)
                    :
                    \lambda \in \R^n,
                    \|\lambda\|_1 = 1
                \right\},
            \]
        which has $\cR_\phi(H\hat\lambda) \geq b$.
    \end{enumerate}
\end{proposition}

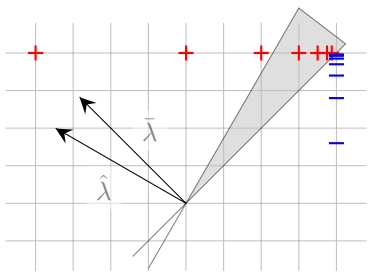
\begin{figure}
    \centering
    \begin{tikzpicture}[scale = 2]
        \draw[step=0.25cm, very thin, lightgray] (-1.2cm,-0.45cm) grid (1.2cm,1.2cm);

        \filldraw [black, fill=lightgray, opacity = 0.5]
                (0,0) -- (60:1.5) -- (45:1.5) -- cycle;
        \draw [black, opacity = 0.5] (60:-0.5) -- (0,0) -- (45:-0.5);
        \draw [
                decoration={markings,mark=at position 1 with {\arrow[ultra thick]{stealth}}},
                    postaction={decorate}
        ] (0,0) -- (135:1);
        \draw [
                decoration={markings,mark=at position 1 with {\arrow[ultra thick]{stealth}}},
                    postaction={decorate}
        ] (0,0) -- (150:1);
        \draw (135:0.5) node[anchor = south west, fill = white, fill opacity= 0.5,
            text opacity = 1.0] {$\bar\lambda$};

        \draw (150:0.5) node[anchor = north east, fill = white, fill opacity=0.5,
            text opacity = 1.0] {$\hat\lambda$};


        \foreach \x in {1,...,7}
        {
            \coordinate (p) at (1- 4 * 2^-\x, 1);
            \draw[thick, red] ($(p) - (0,0.05)$) -- ($(p) + (0,0.05)$);
            \draw[thick, red] ($(p) - (0.05,0)$) -- ($(p) + (0.05,0)$);
        };

        \foreach \x in {1,...,7}
        {
            \coordinate (p) at (1, 1- 1.2 * 2^-\x);
            \draw[thick, blue] ($(p) - (0.05,0)$) -- ($(p) + (0.05,0)$);
        };
    \end{tikzpicture}
    \caption{A bad example for unconstrained linear classification; please
    see \Cref{fact:separable_case_impossibility}.\label{fig:impossibility}}
\end{figure}

A full proof is provided in \Cref{sec:app:impossibility},
but the mechanism is simple enough to appear as a picture.
Consider the linear classification problem in \Cref{fig:impossibility}, which has
positive (``\red{\textbf{+}}'') and negative
(``\blue{\textbf{-}}'') examples along two lines.
Optimal solutions to $\cR_\cL$ are of the
form $c\bar\lambda$, where $\bar\lambda = (-1,+1)$ and $c >0$
(note $\lim_{c\uparrow \infty} \cR_\phi(c\bar\lambda) = \cR_\phi(\SPAN(\cH)) = 0$).
Unfortunately, the positive and negative examples are staggered; as a result,
for any sample, every
max margin predictor $\hat\lambda$,
which is determined solely by the rightmost
``\red{\textbf{+}}'' and uppermost ``\blue{\textbf{-}}'',
will fail to agree with the optimal predictor
on some small region.  A positive probability mass of points fall within this region,
and so, by considering scalings $c\hat\lambda$ as $c\uparrow\infty$, the convex
risk $\cR_\phi$ may be made arbitrary large.

The statement of \Cref{fact:separable_case_impossibility} is encumbered with details
in order to convey the message that not only do such examples exist, they are fairly
benign; indeed, the example depends on the additional regularity of large margin
solutions.   The only difficulty is the lack of any norm constraint on
permissible iterates.

On the other hand, notice that the classification risk $\cR_\cL$ is not only small,
but its empirical counterpart $\cR^m_\cL$ provides a reasonable estimate as $m$ increases.
Furthermore, if the distribution were adjusted slightly so that every $\lambda\in \R^n$
made some
mistake, then these unbounded iterates would fail to exist:
the huge penalty for predictions
very far from correct would constrain the norms of all good predictors.

The preceding paragraph describes the exact strategy of the remainder of the manuscript:
linear classification problems are split into two pieces, one where optimization may
produced unboundedly large iterates with small classification risk, and another piece
where iterates are bounded thanks to the presence of difficult examples.

\section{Hard cores}
\label{sec:hard_cores}

One way to split a linear classification problem into two pieces, one bounded and one
unbounded, is to identify a \emph{hard core} of very difficult instances.
(Note, forms of the hard core have been previously used to study linear classification
\citep{russell_hardcore,mukherjee_rudin_schapire_adaboost_convergence_rate,primal_dual_boosting}.)

\begin{definition}
    Given a linear classification problem $(\cH,\mu)$,
    let $\cD(\cH,\mu)$ denote reweightings of $\mu$ which decorrelate every
    regressor $H\lambda$; that is,
    \[
        \cD(\cH,\mu) :=
        \left\{
            p \in L^1(\mu) : p \geq 0, \forall \lambda\in \R^n\centerdot
            \int y (H\lambda)(x) p(x,y) d\mu(x,y) = 0
        \right\}.
    \]
    Correspondingly, $\cS_\cD(\cH,\mu)$ tracks the supports of these weightings:
    \[
        \cS_\cD(\cH,\mu) := \left\{ [p > 0] : p\in\cD(\cH,\mu)\right\}.
    \]
    A \emph{hard core} $\scrC\subseteq \cX\times\cY$
    for $(\cH,\mu)$ is a maximal element of $\cS_\cD(\cH,\mu)$; that
    is,
    \[
        \scrC \in \cS_\cD(\cH,\mu)
        \qquad
        \textup{and}
        \qquad
        \forall C\in\cS_\cD(\cH,\mu)\centerdot \mu(\scrC \setminus C) \geq 0
        \textup{ and } \mu(C\setminus \scrC) = 0.
    \]
    (``Maximal'', in the presence of measures, will always mean up to
    sets of measure zero.)
\end{definition}

Momentarily it will be established that hard cores split problems in the desired way;
but first, note that hard cores actually exist.

\begin{theorem}
    \label{fact:hard_core:existence}
    Every linear classification problem $(\cH,\mu)$ has a hard core.
\end{theorem}

To prove this, first observe that $\cS_\cD(\cH,\mu)$ is nonempty: it always contains
$\emptyset$, with corresponding reweighting $p(x,y) = 0$.  In order to produce a hard
core, it does not suffice to simply union the contents of $\cS_\cD(\cH,\mu)$, since the
resulting set may fail to be measurable, and it is entirely unclear if a corresponding
$p\in\cD(\cH,\mu)$ can be found.  Instead, the full proof in \Cref{sec:app:hard_cores}
constructs
the hard core via an optimization, and the observation that $\cS_\cD(\cH,\mu)$ is closed
under countable unions.

With the basic sanity check of existence out of the way, notice that
hard cores achieve the goal laid out at the closing of \Cref{sec:impossibility}.
The proof, which is somewhat involved, appears in \Cref{sec:app:hard_cores}.

\begin{theorem}
    \label{fact:partial_ultragordan}
    Let problem $(\cH,\mu)$ and hard core $\scrC$ be given.  The following statements hold.
    \begin{enumerate}
        \item There exists a sequence $\{\lambda_i\}_{i=1}^\infty$
        with $y(H\lambda_i)(x) = 0$ for $\mu$-a.e. $(x,y)\in \scrC$,
        and $y'(H\lambda_i)(x') \uparrow \infty$ for $\mu$-a.e. $(x',y')\in \scrC^c$.
        \item Every $\lambda\in\R^n$ satisfies either
            $\mu(\scrC \cap [y(H\lambda)(x) = 0]) = \mu(\scrC)$
            or $\mu(\scrC \cap [y(H\lambda)(x) < 0]) > 0$.
    \end{enumerate}
\end{theorem}

The first property provides the existence of a sequence which is not only very
good $\mu$-a.e.
over $\scrC^c$, but furthermore does not impact the value of $H\lambda$ over
$\scrC$; that is to say, this sequence can grow unboundedly, and have unboundedly positive
margins over $\scrC^c$, while optimization over $\scrC$ can effectively proceed
independently.  On the other hand, $\scrC$ is difficult: every predictor is either
abstaining $\mu$-a.e., or makes errors on a set of positive measure.

Finally, corresponding to the hard core, it is useful to specialize the definition
of risk to consider regions.
\begin{definition}
    Given a set $C$ (typically $\scrC$ or $\scrC^c$), loss $\phi$, function class
    $\cF$, and any $f\in\cF$, define
    \[
        \cR_{\phi;C}(f) := \int \phi(-y f(x))\1((x,y)\in C)d\mu(x,y),
        \qquad
        \qquad
        \cR_{\phi;C}(\cF) := \inf_{f\in\cF} \cR_{\phi;C}(f),
    \]
    with analogous definitions for $\cR^m_{\phi;C}$, $\cR^m_{\cL;C}$, etc.
\end{definition}

\section{Hard cores and convex risk}
\label{sec:hard_core:true_risk}

The hard core imposes the following structure on $\cR_\phi$.  As provided
by \Cref{fact:partial_ultragordan}, there is a sequence which does arbitrarily well
over $\scrC^c$, without impacting predictions over $\scrC$.  On the other hand, since
mistakes must occur over $\scrC$, convex losses within $\Phi$ will be forced to
avoid large predictors.

\begin{theorem}
    \label{fact:hard_core:true_risk}
    Let problem $(\cH,\mu)$, hard core $\scrC$, and loss $\phi\in\Phi$ be given.
    \begin{enumerate}
        \item
            There exists a sequence $\{\lambda_i\}_{i=1}^\infty$ with
            $y(H\lambda_i)(x) =0$ for $\mu$-a.e. $(x,y)\in\scrC$, and
            $\lim_{i\to\infty} \phi(-y'(H\lambda_i)(x'))=0$ for $\mu$-a.e.
            $(x',y')\in\scrC^c$.
            \label{fact:hard_core:true_risk:1}
        \item
            Let any $\rho>0$ be given.  Then there exists $c_\rho\in\R$ and a set
            $N_\rho$ with $\mu(N_\rho) = 0$ so
            that for every $\lambda\in\R^n$ with $\cR_{\phi;\scrC}(H\lambda) \leq
            \cR_{\phi;\scrC}(\SPAN(\cH)) + \rho$,
            there exists a representation $\lambda'\in\R^n$ with
            $H\lambda = H\lambda'$ over $\scrC\setminus N_\rho$, and
            $\|\lambda'\|_1 \leq c_\rho$.
            \label{fact:hard_core:true_risk:2}
    \end{enumerate}
\end{theorem}

The structural properties of the true convex risk transfer over, with high probability,
to any sampled problem.  Crucially, the various bounds are quantified outside the
probability; that is to say, they do not depend on the sample.

\begin{theorem}
    \label{fact:hard_core:empirical_risk}
    Let problem $(\cH,\mu)$, hard core $\scrC$, and loss $\phi\in\Phi$ be given.
    \begin{enumerate}
        \item With probability 1 over the draw of a finite sample,
            there exists $\lambda\in\R^n$ so that every $(x_i,y_i)\in\scrC^c$
            satisfies $y_i (H\lambda)(x_i) > 0$,
            and every $(x_i',y_i') \in \scrC$ satisfies
            $y_i'(H\lambda)(x_i') = 0$.
            \label{fact:hard_core:empirical_risk:1}
        \item
            Given any empirical suboptimality $\rho >0$, there exist $c>0$ and $b>0$
            so that for any $\delta>0$, with probability at least $1-\delta$
            over a draw of $m$ points where $m_\scrC$, the number of points landing
            in $\scrC$, has bound
            \[
                m_\scrC \geq c^2(\ln(n) + \ln(1/\delta)),
            \]
            then every $\rho$-suboptimal $\lambda\in\R^n$ over the sample
            restricted to $\scrC$,
            meaning
            \[
                \cR^m_{\phi;\scrC}(H\lambda) \leq \cR^m_{\phi;\scrC}(\SPAN(\cH)) +\rho,
            \]
            has a representation $\lambda'$ with $\|\lambda'\|_1 \leq b$
            which has $H\lambda = H\lambda'$ over the sample restricted to $\scrC$,
            and in general $\mu$-a.e. over $\scrC$.
            \label{fact:hard_core:empirical_risk:2}
    \end{enumerate}
\end{theorem}

\section{Deviation inequalities}
\label{sec:deviations}

With the structure of the convex risk in place, the stage is set to establish deviation
inequalities.  These will be stated in terms of both a convex risk $\cR_\phi$,
but also the classification risk $\cR_\cL$.  In order to make this correspondence,
this manuscript relies on standard techniques due to
\citet{zhang_convex_consistency} and \citet*{bartlett_jordan_mcauliffe}.
\begin{definition}
    Let $\fF$ denote the set of 
    measurable functions over $\cX$.
\end{definition}

\begin{proposition}[\citet{bartlett_jordan_mcauliffe}]
    \label{fact:psi:basics}
    Let any $\phi\in\Phi$ be given with $\phi$ differentiable at 0.
    There exists an associated function
    $\psi : [0,1]\to [0,\infty)$ with the following properties.  First,
        for any probability measure $\mu$ and any $f:\cX\to\R$,
        $\psi(\cR_\cL(f) - \cR_\cL(\fF)) \leq \cR_\phi(f) - \cR_\phi(\fF)$.
        Second, the inverse $\psi^{-1}$ exists over $[0,\infty)$, and satisfies
            $\psi^{-1}(r) \downarrow 0$ as $r\downarrow 0$.
\end{proposition}

\begin{definition}
    Given $\phi\in\Phi$, let $\psi$, called the \emph{$\psi$-transform},
    be as in \Cref{fact:psi:basics}.
\end{definition}

The general use of $\psi$ is through its inverse, which provides
\begin{align*}
    \cR_\cL(H\lambda) - \cR_\cL(\fF)
    &\leq \psi^{-1}(\cR_\phi(H\lambda) - \cR_\phi(\fF))
    \\
    &=
    \psi^{-1}\left(\cR_\phi(H\lambda) - \cR_\phi(\SPAN(\cH))
    + \cR_\phi(\SPAN(\cH))- \cR_\phi(\fF)\right).
\end{align*}
Although $\psi^{-1}$ may be unwieldy, it is frequently easy to provide a useful
upper bound.  For instance, the exponential loss has
$\psi^{-1}(r) \leq 2\sqrt r$,
the logistic loss has $\psi^{-1}(r) \leq 4\sqrt r$,
and the hinge loss has $\psi^{-1}(r) = r$
\citep{zhang_convex_consistency,bartlett_jordan_mcauliffe}.

\begin{theorem}
    \label{fact:deviations}
    Let $(\cH,\mu)$, $\scrC$,
    and $\phi\in\Phi$ be given.
    Let a suboptimality tolerance $\rho>0$ be given;
    results will depend on reals $c>0$ and $b>0$ determined by the preceding terms.
    The following statements simultaneously hold with any probability
    $1-\delta$ over the draw of $m$ samples (with $\delta' := \delta/8$ for convenience),
    and any weighting $\lambda\in\R^n$ which is $\epsilon$-suboptimal (with
    $\epsilon \leq \rho$)
    for the corresponding
    surrogate empirical risk problem, meaning
    $\cR_\phi^m(H\lambda) \leq \cR_\phi^m(\SPAN(\cH)) + \epsilon$.
    \begin{enumerate}
        \item
            \label{fact:deviations:1}
            Let $m_\scrC$ and $m_+$ respectively denote the number of samples
            falling into $\scrC$ and $\scrC^c$.  Then
            \begin{align*}
                m_\scrC &\geq m\left(\mu(\scrC)
                 - \sqrt{\ln(1/\delta') / (2m)}
                \right),
                \\
                m_+ &\geq m\left(\mu(\scrC^c)
                 - \sqrt{\ln(1/\delta') / (2m)}
                \right).
            \end{align*}
        \item
            \label{fact:deviations:2}
            The true classification risk over the unbounded portion, $\scrC^c$, has bound
            \begin{equation}
                \!\!\!\!\!\cR_{\cL;\scrC^c}(H\lambda)
                \leq \frac{\epsilon}{\phi(0)}
                +2\sqrt{\frac{2\epsilon(n\ln(2m_++1) + \ln(4/\delta')}{\phi(0)m_+}}
                + \frac{4(n\ln(2m_++1) + \ln(4/\delta')} {m_+}.
                \label{eq:deviations:2:1}
            \end{equation}
            If moreover 
            $\epsilon < \phi(0)/m$, then
            \begin{equation}
                \cR_{\cL;\scrC^c}(H\lambda)
                \leq
                \frac{4(n\ln(2m_++1) + \ln(4/\delta')} {m_+}.
                \label{eq:deviations:2:2}
            \end{equation}
        \item
            \label{fact:deviations:3}
            Suppose
            \[
                m_\scrC \geq c^2 (\ln(n) + \ln(6/\delta')).
            \]
            The true surrogate risk over the unbounded portion has bound
            \begin{equation}
                \cR_{\phi;\scrC}(H\lambda)
                - \cR_{\phi;\scrC}(\SPAN(\cH))
                \leq
                \epsilon
                + \frac {c\left(\sqrt{\ln(n)} + 4\sqrt{\ln(2/\delta')}\right)}{\sqrt {m_{\scrC}}}
                ,
                \label{eq:deviations:3:1}
            \end{equation}
            Additionally, if $\phi$ is differentiable at 0,
            the classification risk has bound
            \begin{align}
            \cR_{\cL;\scrC}(H\lambda) - \cR_{\cL;\scrC}(\fF)
            &\leq
            \psi^{-1}\Bigg(
            \epsilon
            + \frac {c\left(\sqrt{\ln(n)}
            + 4\sqrt{\ln(2/\delta')}\right)}{\sqrt {m_{\scrC}}}
            \notag\\
            &\qquad\qquad
            +\cR_{\phi;\scrC}(\SPAN(\cH))
            - \cR_{\phi;\scrC}(\fF)
            \Bigg).
                \label{eq:deviations:3:2}
            \end{align}
        \item
            Suppose, for simplicity, that
            \[
                m \geq \max\left\{
                    2\ln(1/\delta') / \min\{\mu(\scrC)^2,\mu(\scrC^c)^2\},
                    2c^2(\ln(n) + \ln(1/\delta')) / \mu(\scrC)
                \right\}
            \]
            (where bounds are interpreted to hold trivially when denominators contain 0)
            and additionally that $\epsilon < \phi(0) / m$ and $\phi$ is differentiable
            at 0.
            Then the true classification risk of the full problem
            has bound
            \begin{align*}
                \cR_\cL(H\lambda) - \cR_\cL(\fF)
                &\leq
                \psi^{-1}\Bigg(
                \epsilon
                + \frac {c\sqrt{2}\left(\sqrt{\ln(n)} + 4\sqrt{\ln(2/\delta')}\right)}
                {\sqrt {m\mu(\scrC)}}
                \\
                &\qquad\qquad
                + \cR_{\phi;\scrC}(\SPAN(\cH))
                - \cR_{\phi;\scrC}(\fF)
                \Bigg)
                \\
                &\qquad
                + \frac{8(n\ln(m\mu(\scrC^c)+1) + \ln(4/\delta')} {m\mu(\scrC^c)}.
            \end{align*}
            \label{fact:deviations:4}
    \end{enumerate}
\end{theorem}

\section{Consistency}
\label{sec:consistency}

In order for the predictors to converge to the best choice, near-optimal choices must
be available.  Correspondingly, the first consistency result makes a strong assumption
about the function class, albeit one which may be found in many treatments of
the consistency of boosting (cf. the work of \citet{bartlett_traskin_adaboost}
and \citet[Chapter 12]{schapire_freund_book}).

\begin{theorem}
    \label{fact:consistency:cheating}
    Let $(\cH,\mu)$ and $\phi\in\Phi$ be given with $\phi$ differentiable at 0.
    Suppose $\cR_\phi(\SPAN(\cH)) = \cR_\phi(\fF)$.
    Then there exists a sequence of sample sizes
    $\{m_i\}_{i=1}^\infty\uparrow \infty$, and empirical suboptimality tolerances
    $\{\epsilon_i\}_{i=1}^\infty\downarrow 0$,
    so that every sequence of $\epsilon_i$-suboptimal weightings $\{\lambda_i\}_{i=1}^\infty$
    (i.e., $\cR_\phi^{m_i}(H\lambda_i) \leq \epsilon_i + \cR_\phi^{m_i}(\SPAN(\cH))$)
    satisfies $\cR_\cL(H\lambda_i) \to \cR_\cL(\fF)$ almost surely.
\end{theorem}

This additional assumption is hard to justify in the presence of only finitely many
hypotheses.  To mitigate this, this manuscript follows an approach remarked upon by
\citet[Chapter 12]{schapire_freund_book}: to consider an increasing sequence of classes
which asymptotically grant the desired expressiveness property.

\begin{definition}
    Let a probability measure $\mu$ be given.
    A family of finite hypothesis classes $\{\cH_i\}_{i=1}^\infty$ is called a
    A \emph{linear structural risk minimization family}
    for $\mu$, or simply L-SRM family, if
    for any $\phi\in\Phi$ and tolerance $\epsilon >0$, there exists $j$
    so that $\cR_\phi(\SPAN(\cH_j)) < \cR_\phi(\fF) + \epsilon$.
\end{definition}

The significance of this definition will be clear momentarily, as it grants a stronger
consistency result.  But first notice that straightforward classes satisfy the
L-SRM condition.

\begin{proposition}
    \label{fact:dt:lsrmf}
    Suppose $\cX=\R^d$, and let a probability measure $\mu$ be given
    where $\mu_{\cX}$, the marginal over $\cX$, is a Borel probability measure.
    Let $\cH_i$ denote the collection of decision trees with axis aligned splits
    with thresholds taken from $\{-i,-i+1/i,\ldots,i-1/i,i\}$.
    Then $\{\cH_i\}_{i=1}^\infty$ is an L-SRM family.
\end{proposition}
Proving this fact, as with many classical universal approximation
theorems~\citep{kolmogorov_superposition,cybenko_nn}, relies on basic properties
of continuous functions over compact sets.  In order to reduce to this scenario from
the general scenario of measurable functions $\fF$, Lusin's Theorem is employed,
just as with similar results due to~\citet[Section 4]{zhang_convex_consistency}.

Now that the existence of reasonable L-SRM families is established, note the corresponding
consistency result.

\begin{theorem}
    \label{fact:consistency:lsrm}
    Let probability measure $\mu$ and loss $\phi\in\Phi$ be given with
    $\phi$ differentiable at 0,
    as well as an L-SRM $\{\cH_i\}_{i=1}^\infty$ for $\mu$.
    Then there exists a sequence of sample sizes $\{m_i\}_{i=1}^\infty$,
    a subsequence of classes $\{\cH_{j_i}\}_{i=1}^\infty$,
    and suboptimalities $\{\epsilon_i\}_{i=1}^\infty$,
    so that the every sequence of regressors
    $\{H_{j_i} \lambda_i\}_{i=1}^\infty$  $\epsilon_i$-suboptimal for the corresponding
    empirical problem satisfies
    $\cR_\cL(H_{j_i}\lambda_i) \to \cR_\cL(\fF)$ almost surely.
\end{theorem}

This manuscript is basically saying that
constraining learning at the level of the weak learning oracle
is sufficient for consistency.  Of course, it could be argued that it is more
elegant to instead apply
a regularizer to the objective function (with data-dependent parameter choice),
and permit a powerful weak learning class of infinite size.
But such a discussion is beyond the scope of this manuscript.

\addcontentsline{toc}{section}{References}
\bibliographystyle{plainnat}
\bibliography{ab}
\clearpage
\appendix

\section{Technical Preliminaries}
\begin{lemma}
    \label{fact:phi_basic}
    Let any $\phi\in\Phi$ be given.
    Then $\phi$ is continuous, measurable, and nondecreasing.
    Subgradients exist everywhere, and satisfy $\partial \phi(0) \subseteq \R_{++}$.
    Lastly, the conjugate $\phi^*$ satisfies $\dom(\phi^*)\subseteq \R_+$ and
    $\phi^*(0) = 0$.
\end{lemma}
\begin{proof}
    Since $\phi$ is finite everywhere,
    it is continuous
    \citep[Corollary 10.1.1]{ROC},
    and thus measurable \citep[Corollary 2.2]{folland}.
    Since convex functions are subdifferentiable everywhere
    along the relative interior of their domains (which in this case is just $\R$),
    it follows that $\phi$ has subgradients everywhere \citep[Theorem 23.4]{ROC}.

    If $\phi$ were not nondecreasing, there would exist $x<y$
    with $\phi(x) > \phi(y)$; but that means every subgradient $g\in\partial \phi(x)$
    satisfies
    \[
        \phi(y) \geq \phi(x) + g(y-x),
    \]
    and thus $g < 0$.  But then, for any $z< x$,
    $\phi(z) \geq \phi(x) + g(z-x)$, which in particular contradicts
    $\lim_{z\to-\infty} \phi(z) = 0$ (indeed, it implies
    $\lim_{z\to-\infty} \phi(z) = \infty$), thus $\phi$ is nondecreasing.

    Next, since $\phi$ is nondecreasing, $\partial \phi \subseteq \R_+$.
    However, since $\phi(0) > 0$, it follows that $\partial \phi(0) \subset \R_{++}$,
    since otherwise $\lim_{z\to-\infty}\phi(z) = 0$ would be contradicted.

    Turning to $\phi^*$, first note
    \[
        \phi^*(0) = \sup_z 0\cdot z - \phi(z) = 0.
    \]
    Lastly, since $\phi$ is nondecreasing, then for any $g<0$,
    \[
        \phi^*(g) = \sup_z gz - \phi(z)
        \geq \sup_{z < 0} gz - \phi(z) = \infty.
    \]
    That is to say, $\dom(\phi^*)\subseteq \R_+$.
\end{proof}

\begin{proposition}
    \label{fact:rademacher_risk}
    Let a linear classification problem $(\cH,\nu)$ and loss $\phi\in\Phi$ be
    given.  Then given a bound $b$ on the $l^1$ norm of considered predictors,
    there exists $c\geq \phi(b)$
    so that, for any $\delta > 0$, with probability at least $1-\delta$
    over the draw of $m$ points from $\nu$, every $\lambda\in\R^n$
    with $\|\lambda\|_1\leq b$ satisfies
    \[
        \left|
        \cR_\phi(H\lambda) - \cR^m_\phi(H\lambda)
        \right| \leq \frac {c\left(\sqrt{\ln(n)} + \sqrt{\ln(2/\delta)}\right)}{\sqrt m}.
    \]
\end{proposition}
\begin{proof}
    Let bound $b$ and loss $\phi\in\Phi$ be given.  Define a truncation
    \[
        \hat \phi(z) := \begin{cases}
            \phi(z) & \textup{when $z\leq b$,}\\
            \phi(b) & \textup{otherwise.}
        \end{cases}
    \]
    Since $\phi$ is nondecreasing (cf. \Cref{fact:phi_basic}),
    $\hat\phi(z) \leq \phi(b)$, and furthermore
    $\hat\phi$ is Lipschitz with a constant that may be measured at $b$;
    indeed, since $\phi$ is finite everywhere, it has
    bounded subdifferential sets
    \citep[Theorem 23.4]{ROC}, and thus, taking any $z_1,z_2\in \R$ and supposing
    without loss of generality that $z_1 \leq z_2$,
    \begin{align*}
        |\phi(z_2) - \phi(z_1)|
        &= \phi(z_2) - \phi(z_1)
        \\
        &\leq
        \sup\left\{
            \phi(z_2) - \left(\phi(z_2)
            + \ip{g_2}{z_1 - z_2}\right)
            :
            g_2\in \partial \phi(z_2)
        \right\}
        \\
        &=
        |z_2 - z_1|
        \sup\left\{
            |g_2|
            :
            g_2\in \partial \phi(z_2)
        \right\}
        \\
        &< \infty;
    \end{align*}
    correspondingly, set a Lipschitz constant
    $L_\phi := \sup\{|g| : g\in\partial \phi(b)\}$.

    Note that for every $f\in \SPAN(\cH,b)$, $\sup_{x\in\cX} |f(x)| \leq b$,
    and thus $\cR_\phi(f) = \cR_{\hat\phi}(f)$.
    Lastly, the desired constant $c$, which does not depend on $\delta$, $n$, or $m$,
    will be
    $c := \max\{2L_\phi b \sqrt{2}, \phi(b)\}$.

    Now let a sample of size $m$ be given, and let
    let $R_m(\SPAN(\cH,b))$ denote the Rademacher complexity of $\SPAN(\cH,b)$.
    By properties of Rademacher complexity and a few appeals to McDiarmid's
    inequality,
    \citep[Theorem 3.1, and the proof of Theorem 4.1]{bbl_esaim}, with probability at
    least $1-\delta$ over the draw of this sample,
    \begin{align}
        \sup_{\|\lambda\|_1\leq b}\left|
        \cR_\phi(H\lambda) - \cR^m_\phi(H\lambda)
        \right|
        &=
        \sup_{\|\lambda\|_1\leq b}\left|
        \cR_{\hat\phi}(H\lambda) - \cR^m_{\hat\phi}(H\lambda)
        \right|
        \notag\\
        &\leq 2 L_\phi R_m(\SPAN(\cH,b)) + \sqrt{\frac{2\ln(2/\delta)}{m}}.
        \label{eq:rademacher_risk}
    \end{align}
    Next, by $R_m(\SPAN(\cH,b)) = bR_m(\SPAN(\cH,1)) = bR_m(\cH)$ and an
    appeal to Massart's Finite Lemma \citep[Theorem 3.3]{bbl_esaim}
    \[
        R_m(\SPAN(\cH,b)) \leq \sqrt{\frac{2\ln(n)}{m}}.
    \]
    Plugging this into \cref{eq:rademacher_risk} and recalling
    the choice
    $c = \max\{2L_\phi b \sqrt{2}, \phi(b)\}$, the result follows.
\end{proof}

\begin{lemma}
    \label{fact:convex:increasing}
    Let $S\subset \R$ and convex $f:S\to\R$ be given.
    If $x,y\in S$ are given with $x < y$ and $f(x)<f(y)$,
    then for every $S\ni z > y$, $f(y) < f(z)$.
\end{lemma}
\begin{proof}
    Write $y$ as a combination of $x$ and $z$:
    \[
        y = x \left(\frac{z-y}{z-x}\right)
        + z \left(\frac{y-x}{z-x}\right).
    \]
    By convexity and $f(y)>f(x)$,
    \begin{align*}
        f(y) \left(\frac{z-y}{z-x}\right)
        + f(z) \left(\frac{y-x}{z-x}\right)
        &>
        f(x) \left(\frac{z-y}{z-x}\right)
        + f(z) \left(\frac{y-x}{z-x}\right)
        \\
        &\geq
        f\left(
        x \left(\frac{z-y}{z-x}\right)
        + z \left(\frac{y-x}{z-x}\right)
        \right)
        \\
        &= f(y).
    \end{align*}
    Rearranging and using $x<y$, it follows that $f(y) < f(z)$.
\end{proof}

\section{Convexity properties of $\cR_\phi$}

\begin{lemma}
    \label{fact:intphi_basic}
    Let finite measure $\nu$ and $\phi\in\Phi$ be given.
    Then the function
    \[
        L^\infty(\nu) \ni q \quad \mapsto \quad \int \phi(q) \in \R
    \]
    is well-defined, convex, and lower semi-continuous.  Next,
    $\left(L^\infty(\nu)\right)^*$ can be written as the direct
    sum of two spaces, one being $L^1(\nu)$; for any $p \in
    \left(L^\infty(\nu)\right)^*$, let $p_1 + p_2$ be the corresponding
    decomposition (with $p_1 \in L^1(\nu)$).
    With this notation, $\int q p_2 = 0$ for any $q\in L^\infty(\nu)$; furthermore,
    the Fenchel conjugate to the above map is
    \[
        \left(L^\infty(\nu)\right)^* \ni p \quad \mapsto \quad \int \phi^*(p_1),
    \]
    which is again well-defined, convex, and lower semi-continuous.
    Lastly, the subdifferential set to the first map may be obtained by simply
    passing the subdifferential operator through the integral,
    \begin{align*}
        \partial\left(\int \phi\right)(q)
        &= \left\{ p\in \left(L^\infty(\nu)\right)^*
        : p_1 \in \partial \phi(q)\ \nu\textup{-a.e.}\right\}
        .
    \end{align*}
\end{lemma}

\begin{proof}
    The proof will proceed with heavy reliance upon results due to \citet{roc_conv_int_2}.
    To start, note that $\phi$, being convex and continuous (cf. \Cref{fact:phi_basic}),
    is a \emph{normal convex integrand} \citep[Lemma 1]{roc_conv_int_2}.

    Let $Z : \cX \to \R$ denote the zero map, i.e. $Z(x) = 0$ everywhere.
    Note that $\phi\circ Z \in L^1(\nu)$, and similarly
    $\phi^* \circ Z \in L^1(\nu)$ (since $\phi(0) = 0$; cf. \Cref{fact:phi_basic});
    these facts provide the conjugacy formula
    \begin{equation}
        \left(\int \phi\right)^*(p) =
        \int \phi^*(p_1) + \sup\left\{
            p_2(q) : q\in L^\infty(\nu), \int \phi(q) < \infty
        \right\},
        \label{eq:intphi_basic:1}
    \end{equation}
    where the decomposition $p=p_1+p_2$ is as in the \namecref{fact:intphi_basic}
    statement \citep[Theorem 1]{roc_conv_int_2}.

    Next, notice that $\dom(\int \phi) = L^\infty(\nu)$; in particular,
    given any $q\in L^\infty(\nu)$,
    \[
        \int \phi(q) \leq \int \phi(\|q\|_\infty)
        = \phi(\|q\|_\infty) \nu(\cX,\cY) < \infty.
    \]

    As such, consider an arbitrary $p_2$ and $q\in L^\infty(\nu)$.  Since $p$ is
    a continuous linear functional on $L^\infty(\nu)$, then so is $p_2$ (otherwise
    the formula $p = p_1+p_2$ would not make sense).  Next, as stated by
    \citet[introduction to Section 2]{roc_conv_int_2}, it is possible to choose
    sets $S_k$ with $\nu(S_k^c)< 1/k$, and $p_2(q) = 0$ over every $S_k$
    and $q\in L^1(\nu)$.
    Now define $U_k = \cup_{i \leq k} S_i$.
    By continuity of measures from below \citep[Theorem 1.8c]{folland},
    $\nu(U_k) \uparrow \nu(\cX\times \cY)$.  As such, by the dominated convergence
    theorem \citep[Theorem 2.25]{folland}, and setting $U_0 = \emptyset$,
    \begin{align*}
        \int p_2 q
        &= \int_{\cup_{k=1}^\infty U_k} p_2 q
        \\
        &= \sum_{k=1}^\infty \int_{U_k \setminus U_{k-1}} p_2 q
        \\
        &= 0.
    \end{align*}
    That is to say, the supremum term in \cref{eq:intphi_basic:1} is simply zero;
    plugging this back into \cref{eq:intphi_basic:1}, the desired conjugacy relation
    follows.  Note that the same result, due to \citet[Theorem 1]{roc_conv_int_2},
    provides the integrals are well-defined, and moreover that the pair of conjugate
    functions are both convex and lower semi-continuous (as a consequence of being
    mutually conjugate).  Lastly, the above derivation has established that $\int \phi$
    is finite over $L^\infty(\nu)$, but it is possible that $\int \phi^*$ is infinite,
    even over $L^1(\nu)$ (i.e., and not just over $(L^\infty(\nu))^*$).

    For the subdifferential relation, a related resulted by 
    \citet[Corollary 1A]{roc_conv_int_2} provides that $(L^\infty(\nu))^* \ni p
    \in \partial (\int \phi)(q)$ (for some $q\in L^\infty(\nu)$) precisely when
    $p_1 \in \partial \phi(q)$ $\nu$-a.e., and the supremum in
    \cref{eq:intphi_basic:1} is attained for $p_2$ at $q$.
    It was already established that the supremum is always zero,
    as is $p_2(q)$, and the result
    follows.
%
%
%
%
\end{proof}

\begin{corollary}
    \label{fact:intphi_H_basic}
    Let a finite measure $\nu$ and $\phi\in\Phi$ be given.  The function
    \[
        \R^n \ni \lambda\quad \mapsto\quad \int \phi(-y(H\lambda)(x)) d\nu(x,y) \in \R
    \]
    is convex and continuous.
\end{corollary}
\begin{proof}
    Note that
    \[
        \lambda \mapsto -y(H\lambda)x
    \]
    is a bounded linear operator (and thus continuous), and the latter object,
    taken as a function over $\cX\times \cY$, is within $L^\infty(\nu)$.
    Combined with the
    lower semi-continuity and convexity of $\int \phi$ as per \Cref{fact:intphi_basic},
    it follows that the the map in question is convex and lower semi-continuous.
    Since it is finite everywhere, it is in fact continuous \citep[Corollary 7.2.2]{ROC}.
\end{proof}

\begin{lemma}
    \label{fact:duality}
    Let a linear classification problem $(\cH,\nu)$ and any $\phi\in\Phi$ be given.
    Then
    \[
        \inf\left\{
            \int \phi(-y(H\lambda)x)d\nu(x,y) : \lambda\in\R^n
        \right\}
        = \max\left\{
            \int -\phi^*(p) : \max\{p,0\} \in \cD(\cH,\nu)
        \right\},
    \]
    where the $\max$ is taken element-wise.  Furthermore, if a primal optimum
    $\bar\lambda$ exists, then there is a $\bar p\in\cD(\cH,\nu)$ with $\bar
    p(x,y) \in \partial \phi(-y(H\lambda)x)$ $\nu$-a.e.
\end{lemma}
\begin{proof}
    For convenience, define the linear operator
    \[
        (A\lambda)(x,y) := -y(H\lambda)x.
    \]
    Note that $A$ is a bounded linear operator, and furthermore has transpose
    \[
        A^\top p := \sum_{i=1}^n \bfe_i
        \int -y h_i(x) p(x,y) d\nu(x,y)
    \]
    (this follows by checking $\ip{A\lambda}{p} = \ip{\lambda}{A^\top p}$
    for arbitrary $\lambda\in\R^n$ and $p\in (L^\infty(\nu))^*$,
    which entails the formula above provides the unique transpose
    \citep[Theorem 4.10]{rudin_functional}.)

    Consider the following two Fenchel problems:
    \begin{align*}
        p &:= \inf \left\{
            \int \phi(A\lambda) + \ip{0}{\lambda} : \lambda \in \R^n
        \right\},
        \\
        d &:= \sup\left\{
            -\int \phi^*(p_1) - \iota_{\{0\}}(A^\top p)
            : p\in \left(L^\infty(\nu)\right)^*
        \right\},
    \end{align*}
    where $\iota_ {\{0\}}$ is the indicator for the set $\{0\}$,
    \[
        \iota_{\{ 0\} }(\lambda) =
        \begin{cases}
            0 &\textup{when } \lambda=0, \\
       \infty &\textup{otherwise,}
        \end{cases}
    \]
    and is the conjugate to $\ip{0}{\cdot}$; additionally, $p_1$ is as discussed in
    the statement of \Cref{fact:intphi_basic}.  To show $p=d$ and thus prove the desired
    result, an appropriate Fenchel duality rule will be applied
    \citep[Corollary 2.8.5 using condition (vii)]{zalinescu}.

    To start, note that $\int\phi$ and $\int\phi^*$ are conjugates,
    as provided by \Cref{fact:intphi_basic}.
    Next, also from \Cref{fact:intphi_basic},
    $\int \phi$ is finite everywhere over $L^\infty(\nu)$.
    As a result,
    \begin{align*}
        A\dom(\ip{0}{\cdot}) - \dom(\int \phi)
        = A\dom(\ip{0}{\cdot}) -L^\infty(\nu)
        = L^\infty(\nu).
    \end{align*}
    The significance of this fact is that it will act as the constraint qualification
    granting $p=d$.

    Lastly, $\R^n$ and $L^\infty(\nu)$ are Banach and thus Fr\'echet spaces.
    As such, all conditions necessary for Fenchel duality are met
    \citep[Corollary 2.8.5 using condition (vii)]{zalinescu}, and it follows that $p=d$
    as desired, with attainment in the dual.

    The next goal is to massage this duality expression into the one appearing
    in the \namecref{fact:duality} statement.  To start, as provided by
    \Cref{fact:intphi_basic}, $\int q p_2= 0$ for any $q\in L^\infty(\nu)$, and
    in particular $A^\top p_2 = 0$; consequently, $p_2$ has no effect on either
    term in the dual objective, and the domain of the dual may be
    restricted to $L^1(\nu)$.

    Next, \Cref{fact:phi_basic} grants $\dom(\phi^*)\subseteq \R_+$,
    and so the domain of the dual problem may be safely restricted to $p \geq 0$
    $\nu$-a.e. (since $0$ is always dual feasible, and $\nu([p < 0]) > 0$ entails an
    objective value of $-\infty$).  By the form of $A^\top$,
    $\iota_{0}(A^\top p)$ is finite iff
    \[
        \int y h(x) p(x,y) d\nu(x,y) = 0
    \]
    for all $h$; it follows that $\iota_{0}(A^\top p)$ is finite iff
    \[
        \int (A\lambda)(x,y) p(x,y) d\nu(x,y) = 0
    \]
    for all $\lambda\in\R^n$.  Combining these facts, an equivalent form for the
    dual problem is
    \[
        \max\left\{-\int \phi^*(p) : \max\{p,0\} \in \cD(\cH,\nu)\right\},
    \]
    just as in the statement of the \namecref{fact:duality}.

    Lastly, the Fenchel duality rule invoked above, as presented by \citet{zalinescu},
    also provides that a primal optimum $\bar \lambda$ exists iff there is a
    $p' \in (L^\infty(\nu))^*$
    with $-A^\top p' \in \partial (\ip{0}{\cdot})(\bar\lambda) = 0$
    and $p' \in \partial (\int \phi)(A\bar\lambda)$.  The first part simply states
    that $\max\{p', 0\}\in\cD(\cH,\nu)$ as above.
    The second part, when combined with the subdifferential
    rule of \Cref{fact:intphi_basic}, gives
    $p_1' \in \partial \phi(A\bar \lambda)$ $\nu$-a.e.  To obtain the desired statement,
    set $\bar p := \max\{p_1', 0\}$, which satisfies all desired properties.
\end{proof}

\section{Structure of $\cR_\phi$ over $\cS_\cD(\cH,\mu)$}

The following \namecref{fact:somewhat_stiemke} leads to a number of properties
presented in \Cref{sec:hard_cores,sec:hard_core:true_risk}; it is easiest to prove
them at once, as a ring of implications.

\begin{theorem}
    \label{fact:somewhat_stiemke}
    Let a linear classification problem $(\cH,\mu)$ and a set $D$ be given.
    The following statements are equivalent.
    \begin{enumerate}
        \item
            For every $\lambda\in\R^n$,
            either $\mu(D \cap [y(H\lambda)x = 0]) = \mu(D)$
            or $\mu(D \cap [y(H\lambda)x < 0]) > 0$.
            \label{fact:somewhat_stiemke:1}
        \item Given any $\rho$, there exists a bound $b$ and a null set $N\subseteq
            \cX\times\cY$ (i.e., $\mu(N) = 0$) so that for every $\rho$-suboptimal
            weighting $\hat\lambda$ over $D$, meaning any weighting satisfying
            \[
                \cR_{\phi;D}(H\hat\lambda) \leq \cR_{\phi;D}(\SPAN(\cH)) + \rho,
            \]
            there exists $\lambda'$ with $\|\lambda\|_1 \leq b$ and
            $H\hat\lambda = H\lambda'$ over $D\setminus N$.
            \label{fact:somewhat_stiemke:2}
        \item $D\in\cS_\cD(\cH,\mu)$.
            \label{fact:somewhat_stiemke:3}
    \end{enumerate}
\end{theorem}

The following structural \namecref{fact:almost_kernel} is crucial.

\begin{lemma}
    \label{fact:almost_kernel}
    Let $(\cH,\mu)$ and a set $D$ be given.
    Define the set
    \[
        \cK := \{\lambda\in\R^n : y(H\lambda)x = 0\textup{ for $\mu$-a.e. } (x,y)\in D
        \}.
    \]
    The following statements hold.
    \begin{enumerate}
        \item $\cK$ is a subspace.
            \label{fact:almost_kernel:1}
        \item There exists a set $N$ with $\mu(N)=0$ so that,
            for any for any $\lambda\in\R^n$, the orthogonal projection
            $\lambda \mapsto \lambda^\perp\in\cK^\perp$
            satisfies $H\lambda = H\lambda^\perp$ everywhere over $D\setminus N$.
            \label{fact:almost_kernel:2}
        \item There exists a constant $c >0$ so that, for any $\lambda\in\R^n$
            with $\mu(D \cap [H\lambda \neq 0]) > 0$,
            $\|H\lambda\|_{L^\infty(\mu_D)} / \|\lambda^\perp\|_1 > c$,
            where $L^\infty(\mu_D)$ is the $L^\infty$ metric with respect to the
            measure defined by $\mu_D(S) = \mu(D\cap S)$ for any measurable set
            $S$.
            \label{fact:almost_kernel:3}
    \end{enumerate}
\end{lemma}

\begin{proof}
    (\Cref{fact:almost_kernel:1})
    Direct from its construction, $\cK$ is a subspace.  Crucially, this means that
    $\cK^\perp$ is also a subspace, and the orthogonal projection $\lambda\mapsto
    \lambda^\perp$ exists.

    (\Cref{fact:almost_kernel:2})
    Given the subspace pair $\cK$ and $\cK^\perp$, for any $\lambda\in\R^n$,
    there exists the decomposition $\lambda \mapsto \lambda^\cK + \lambda^\perp$,
    where $\lambda^\perp\in\cK^\perp$.  By definition, $H\lambda^\cK= 0$ $\mu$-a.e.
    over $D$, and thus $H\lambda = H\lambda^\perp$ $\mu$-a.e. over $D$.

    Now let $Q$ be any countable dense subset of $\R^n$.  For each $\lambda_i\in Q$,
    define $N_i:= [H\lambda_i \neq H\lambda_i^\perp]$, where the above provides
    $\mu(N_i) = 0$.  Set $N := \cup_i N_i$, which is measurable since it is a countable
    union, and moreover $\mu(N) = 0$ by $\sigma$-additivity.  It will now be argued
    that the projections onto $\cK^\perp$ give equivalences over $D\setminus N$.

    To this end, let any $\lambda\in\R^n$, any $(x,y)\in D\setminus N$,
    and any $\tau >0$ be given.
    Since $Q$ is a countable dense subset of $\R^n$, there exists $\lambda_i\in Q$
    with $\|\lambda_i - \lambda\|_1 \leq \tau/2$.
    Now let $P^\perp$ denote the orthogonal projection operator onto $\cK^\perp$;
    then
    \begin{align*}
        0 \leq |(H\lambda)(x) - (H\lambda^\perp)(x)|
        &=
        |(H\lambda)(x) - (HP^\perp\lambda)(x)|
        \\
        &=
        |(H(\lambda-\lambda_i + \lambda_i)(x)
        - (HP^\perp(\lambda-\lambda_i+\lambda_i))(x)|
        \\
        &\leq
        |(H\lambda_i)(x) - (H\lambda_i^\perp)(x)|
        + |H(\lambda-\lambda_i)(x)| + |HP^\perp(\lambda-\lambda_i)(x)|
        \\
        &\leq
        |0| + \|H\|_\infty \|\lambda-\lambda_i\|_1
        + \|H\|_\infty \|P^\perp\|_\infty\|\lambda-\lambda_i\|_1
        \\
        &\leq
        0
        + \tau/2 + \tau/2
        = \tau.
    \end{align*}
    Taking $\tau\downarrow 0$, it follows that $H\lambda = H\lambda^\perp$ over
    $D\setminus N$.

    (\Cref{fact:almost_kernel:3})
    For the final part, if every $\lambda\in\R^n$ has $\mu(D\cap [H\lambda \neq 0]) = 0$,
    there is nothing to show, so suppose there exists $\lambda\in\R^n$
    with $\mu_D([H\lambda\neq 0]) > 0$.  Consider the optimization problem
    \[
        \inf\left\{
            \frac {\|H\lambda\|_{L^\infty(\mu_D)}}{\|\lambda^\perp\|_1}
            : \lambda \in \R^n, \mu_D([H\lambda \neq 0]) > 0
        \right\}
        =
        \inf\left\{
            \|H\lambda\|_{L^\infty(\mu_D)}
            : \lambda \in \cK^\perp, \|\lambda\|_1 = 1
        \right\}.
    \]
    The latter is a minimization of a continuous function over a nonempty compact set,
    and thus attains a minimizer $\bar \lambda$.  But $\bar\lambda\in \cK^\perp$
    and $\|\bar\lambda\|_1=1$, thus $\|H\bar\lambda\|_{L^\infty(\mu_D)} > 0$.
    The result follows with $c := \|H\bar\lambda\|_{L^\infty(\mu_D)} >0$.
\end{proof}

\begin{proof}[Proof of \Cref{fact:somewhat_stiemke}]
    (\Cref{fact:somewhat_stiemke:1} $\implies$ \Cref{fact:somewhat_stiemke:2}.)
    Let $\rho$ be given, and let $N$ be the set, as provided by
    \Cref{fact:almost_kernel}, so that every $\lambda\in\R^n$ has
    $H\lambda = H\lambda^\perp$ everywhere on $D\setminus N$.
    Suppose contradictorily that the remainder of the desired statement is false;
    one way to say this is that there exists a sequence $\{\lambda_i\}_{i=1}^\infty$
    so that every equivalent representation over $D\setminus N$ (i.e., $H\lambda_i=
    H\lambda'_i$ over this set) has $\sup_i \|\lambda_i'\|_1 = \infty$, but
    $\cR_{\phi;D}(H\lambda_i) \leq \cR_{\phi;D}(\SPAN(\cH)) + \rho$.
    (It can be taken without loss of generality that $\lambda_i \neq 0$ for
    every $i$.)

    To build the contradiction, choose representation $\lambda_i^\perp$,
    which satisfies $H\lambda_i^\perp = H\lambda_i$ over $D\setminus N$
    via \Cref{fact:almost_kernel}.
    Note that $\{\lambda_i^\perp/\|\lambda_i^\perp\|_1\}_{i=1}^\infty$ lies in a compact
    set (the unit $l^1$ ball), and thus let $\LL{2}{i}$ be a subsequence
    with $\LL{2}{i}/\|\LL{2}{i}\|_1 \to \bar \lambda \in \R^n$.
    Since the assumed contradiction was that no representation is bounded,
    $\LL{2}{i}$ is unbounded;
    since there exists a $c>0$ with $\|H\LL{2}{i}\|_{L^\infty(\mu_D)}/\|\LL{2}{i}\|_1\geq c$
    (cf. \Cref{fact:almost_kernel}), it follows
    by continuity of $H$ and norms
    that
    $\|H\bar\lambda\|_{L^\infty(\mu_D)}\geq c$,
    and in particular
    $\mu(D\cap [y(H\bar\lambda)x \neq 0]) > 0$.

    By assumption (i.e., by \Cref{fact:somewhat_stiemke:1}),
    since $\mu(D\cap [y(H\bar\lambda)x \neq 0])>0$,
    then $\mu(D\cap[y(H\bar\lambda)x < 0]) > 0$;
    for convenience, define the set $P:= [y(H\bar\lambda)(x) < 0]$.
    Thus, for any $\lambda\in\R^n$,
    taking any $g\in\partial\phi(0)$
    (note $g > 0$ via \Cref{fact:phi_basic}),
    \begin{align}
        &\lim_{t\to\infty} \frac {\int_D \phi(-y(H(\lambda + t\bar\lambda))(x)) - \int_D\phi(-y(H\lambda)(x))}{t}
        \notag\\
        &\geq
        \lim_{t\to\infty} \frac {\int \phi(-y(H(\lambda + t\bar\lambda))(x))\1((x,y)\in D\cap P) - \int_D\phi(-y(H\lambda)(x))}{t}
        \notag\\
        &\geq
        \lim_{t\to\infty} \frac {\int (\phi(0) + g(-y(H(\lambda + t\bar\lambda))(x)))\1((x,y)\in D\cap P) - \int_D\phi(-y(H\lambda)(x))}{t}
        \notag\\
        &=
        g \int -y(H\bar\lambda)(x) \1((x,y)\in D \cap P)
        \notag\\
        &\qquad
        + 
        \lim_{t\to\infty} \frac {\int (\phi(0) + g(-y(H\lambda)(x)))\1((x,y)\in D\cap P) - \int_D\phi(-y(H\lambda)(x))}{t}
        \notag\\
        &> 0.\label{eq:ultragordan:0coercivity}
    \end{align}
    The above statement shows that $\int_D\phi$ eventually grows in direction
    $H\bar\lambda$, and in particular must exit the desired $\rho$-sublevel set
    \[
        C_\rho := \{\lambda \in \R^n : \cR_{\phi;D}(H\lambda) \leq \cR_{\phi;D}(\SPAN(\cH)) + \rho\}.
    \]
    To develop the contradiction, it will be shown that the construction of $\bar\lambda$
    indicates it should be in this sublevel set $C_\rho$; the proof will
    be similar to one due to
    \citet[Proposition A.2.2.3]{HULL}.

    Since $\int \phi$ and $\int_D \phi$ are convex and lower semi-continuous (cf. \Cref{fact:intphi_basic}),
    sublevel sets, in particular $C_\rho$, are closed convex sets.
    By construction of $\bar\lambda$,
    \[
        H\lambda_j+ t H\bar\lambda = \lim_{i\to\infty}\left(
        (1- \frac t {\|\LL{2}{i}\|_1})H\lambda_j  +\frac t {\|\LL{2}{i}\|_1} H\LL{2}{i}
        \right) \in C_\rho.
    \]
    This holds for all $t>0$, but since $H\bar\lambda \neq 0$,
    \cref{eq:ultragordan:0coercivity}
    forces $H\lambda_i + tH\bar\lambda$ to leave any sublevel set (for
    sufficiently large $t$),
    and in particular $C_\rho$, a contradiction.

    (\Cref{fact:somewhat_stiemke:2} $\implies$ \Cref{fact:somewhat_stiemke:3}.)
    Choose $\phi := \exp \in \Phi$, and a minimizing sequence $\LL{1}{i}$ for
    $\cR_{\phi;D}$, meaning $\cR_{\phi;D}(H\LL{1}{i}) \to \cR_{\phi;D}(\SPAN(\cH))$.
    Choose any suboptimality $\rho$, and produce $\LL{2}{i}$ by removing
    all $\LL{1}{j}$ with $\cR_{\phi;D}(H\LL{1}{j}) > \cR_{\phi;D}(\SPAN(\cH))
    + \rho$ (this procedure must be possible, since otherwise $\{\LL{1}{i}\}_{i=1}^\infty$
    is not a minimizing sequence).
    By the assumed statement,
    there exists $b>0$ and a null set $N$
    so that each $\LL{2}{i}$ may be replaced with
    $\LL{3}{i}$, where $\|\LL{3}{i}\|_1\leq b$, and $H\LL{2}{i} = H\LL{3}{i}$ over
    $D\setminus N$, which in particular means $\LL{3}{i}$ is also a minimizing
    sequence.  But this is now a minimizing sequence lying within a compact set,
    so, perhaps by passing to a subsequence $\LL{4}{i}$, it has a limit
    $\bar\lambda\in \R^n$.  Since $\lambda \mapsto \int\phi(-y(H\lambda)x)$ is
    continuous (cf. \Cref{fact:intphi_H_basic}), it follows that $\bar\lambda$
    attains the desired infimal value.

    Applying the duality relation in \Cref{fact:duality}
    to $\cR_{\phi,D}$ (i.e., using the measure $\nu = \mu_D$, meaning $\nu(S) = \mu(D\cap S)$ for
    any measurable set $S$),
    the existence of a primal
    minimum $\bar\lambda$ grants the existence of a dual maximum $\bar p$
    satisfying $\bar p \in \cD(\cH,\nu)$, and moreover
    \[
        \bar p(x,y) \in  \partial \phi(-y(H\bar\lambda)x) = \exp(-y(H\bar\lambda)x)
    \]
    $\nu$-a.e.
    As such, the choice $p'(x,y) := \exp(-y(H\bar\lambda)(x))$ satisfies $p' := \bar p$
    $\nu$-a.e., and thus $p' \in \cD(\cH,\nu)$; moreover $p' > 0$ everywhere,
    since $\exp > 0$ everywhere.

    This reweighting $p'$ was with respect to $\nu$, so to finish, define
    $p^*(x,y) := p'(x,y) \1((x,y) \in D)$.  By construction, $[p^* > 0] = D$.
    Finally, given any $\lambda \in \R^n$,
    \begin{align*}
        \int y (H\lambda)(x)p^*(x,y)d\mu(x,y)
        &=
        \int y (H\lambda)(x)p'(x,y)\1((x,y)\in D)d\mu(x,y)
        \\
        &=
        \int y (H\lambda)(x)p'(x,y)d\mu_D(x,y)
        \\
        &= 0.
    \end{align*}
    It follows that $p^* \in \cD(\cH,\mu)$, and that $D \in \cS_\cD(\cH,\mu)$.

    (\Cref{fact:somewhat_stiemke:3} $\implies$ \Cref{fact:somewhat_stiemke:1}.)
    Let $p\in\cD(\cH,\mu)$ with $D = [p > 0]$ be given,
    and take any $\lambda\in\R^n$
    satisfying
    $\mu(D \cap [y(H\lambda)x > 0]) > 0$.
    But notice then, since $p$ decorrelates $H\lambda$,
    \begin{align*}
        0 &= \int p(x,y) y (H\lambda) (x) d\mu(x,y)
        \\
        &= \int_{D,y(H\lambda)(x) > 0} p(x,y) y(H\lambda) (x) d\mu(x,y)
        +\int_{D,y(H\lambda)(x) < 0} p(x,y) y(H\lambda) (x) d\mu(x,y).
    \end{align*}
    From this it follows that
    \begin{align*}
        -\int_{D,y(H\lambda)(x) < 0} p(x,y) y(H\lambda) (x) d\mu(x,y)
        = \int_{D,y(H\lambda)(x) > 0} p(x,y) y(H\lambda) (x) d\mu(x,y)
        >0,
    \end{align*}
    where the inequality follows from $\mu(D\cap[y(H\lambda)(x)>0]) > 0$
    \citep[Proposition 2.23(b)]{folland}.  The result follows.
\end{proof}

\section{Deferred material from \Cref{sec:setup}}
\label{sec:app:setup}

In order to invoke standard results for gradient descent, this proof will
use material from \Cref{sec:hard_core:true_risk} to establish the existence of
minimizers.  Although those results appear later in the text, they do not
in turn depend on the material here.

\begin{proof}[Proof of \Cref{fact:opt_oracles_exist}]
    Suppose $\cH$, a sample of size $m$, and suboptimality $\rho>0$ are given as
    specified.  Before proceeding, note briefly that the results invoked below
    --- those demonstrating $\cO(\textup{poly}(1/\rho))$ iterations suffice ---
    neglect to provide a mechanism
    to stop the algorithms, and thus provide a proper oracle.  But this may be
    accomplished by measuring duality gap, for instance by specializing
    the duality relation in \Cref{fact:duality} to the empirical measure.

    First suppose $\phi$ is Lipschitz continuous, attains its infimum, and subgradient
    descent is employed.  Notice that $\cR^m_\phi \circ H$ is also Lipschitz
    continuous (since $H$ is a bounded linear operator), so if it can be shown
    that the infimum is attained, the standard analysis of subgradient descent
    may be applied, which in particular grants a $\cO(1/\rho^2)$ convergence
    rate when a step size of $\cO(1/\sqrt t)$ is employed, where $t$ indexes the iterations
    \citep[Theorem 3.2.2 and subsequent discussion on step sizes]{nesterov_book}.
    To finish, it must be shown that the infimum is attained.

    To this end, let $\mu_m$ be the empirical measure of the training sample,
    and let $\scrC$ be a corresponding hard core.
    By \Cref{fact:hard_core:true_risk}, since $\mu_m$ is now a discrete measure,
    a single weighting $\lambda_0\in\R^n$ can be extracted out with $y(H\lambda_0)(x) > 0$
    over $\scrC^c$ and $y(H\lambda_0)(x) = 0$ over $\scrC$.
    Also by \Cref{fact:hard_core:true_risk}, every 1-suboptimal predictor to $\cR^m_\phi$
    has a representation
    which lies in a compact set; thus, minimizing sequence lies in the compact set, and a
    minimizer $\bar\lambda_0$ exists.  To finish, since $\lim_{z\to-\infty}\phi(z) = 0$ and
    $\phi$ attains its infimum, necessarily there is a $b$ with $\phi(z) =0$ for $z \leq b$.
    As such, it follows that
    \[
        \lambda' := \bar \lambda + \lambda_0 \left(
        \frac
        {z + \|H\bar\lambda\|_\infty}
        {\min\{ |y_i (H\lambda_0)(x_i)|: (x_i,y_i) \in \scrC^c\}}
        \right)
    \]
    is an optimum to the full problem.  First, it is zero over $\scrC^c$, since
    for any $(x,y)\in\scrC^c$,
    \begin{align*}
        y(H\lambda')(x)
        &= y(H\bar\lambda)(x)
        + y(H \lambda_0)(x) \left(
        \frac
        {z + \|H\bar\lambda\|_\infty}
        {\min\{ |y_i (H\lambda_0)(x_i)|: (x_i,y_i) \in \scrC^c\}}
        \right)
        \\
        &\geq
        -\|H\bar\lambda\|_\infty
        + (z + \|H\bar\lambda\|_\infty),
    \end{align*}
    and the choice of $z$ (i.e., $\phi(-y(H\lambda')(x)) = 0$).
    Next, $\lambda'$ is equivalent to
    to $\bar \lambda$ over $\scrC$.
    Finally, if there exists some $\lambda^*$ which achieves a lower objective value than
    $\lambda'$, necessarily it would
    be better than $\bar \lambda$ over $\scrC$, contradicting optimality of $\bar \lambda$.
    In particular, the infimum is attained, and the proof for this choice of $\phi$ is complete.

    Now suppose that $\phi$ is in the convex cone generated by the logistic
    and exponential losses; if it can be shown that
    $\phi$ is within $\bG$, a class of losses known to possess $\cO(1/\rho)$ convergence
    rates for boosting \citep[Definition 19, Theorem 21, Theorem 23,
    Theorem 27]{primal_dual_boosting}, then the result follows.

    To this
    end, first notice that $\bG$ is a cone: given any $c>0$
    and $g\in\bG$ with certifying constants $\eta,\beta$,
    then $c g\in \bG$ with the exact same constants.  Since
    the exponential and logistic losses are within $\bG$
    \citep[Remark 46]{primal_dual_boosting}, then so are all rescalings.

    To finish, let $\phi_1$ and $\phi_2$ respectively denote the logistic and
    exponential losses, and let any $c_1,c_2 >0$ be given; if it can be shown that
    $c_1\phi_1 + c_2\phi_2 \in \bG$, then combined with the earlier cases,
    the proof is complete.  First note that
    \[
        \sum_{i=1}^m (c_1 \phi_1(x_i) + c_2 \phi_2(x_i))
        \leq m(c_1 \phi_1(0) + c_2\phi_2(0))
    \]
    implies
    \[
        \forall i\centerdot
        x_i \leq \ln\left( \frac{m(c_1\phi_1(0) + c_2\phi_2(0))}{c_2}\right);
    \]
    henceforth define $c:= m(c_1\phi_1(0) + c_2\phi_2(0)) / c_2$, and as per
    the definition of $\bG$, the constants $\eta$ and $\beta$ must be
    established under the assumption $x\leq \ln(c)$.

    For any $x \in (-\infty,\ln(c)]$, since $\ln$ is convex, there is a secant lower bound
    \[
        \ln(1+e^x)
        \geq \left(
        \frac{\ln(1+c) - 0} {c-0}
        \right)e^x;
    \]
    as usual, there is also the upper bound $\ln(1+e^x) \leq e^x$.

    As such, for any $x\in (-\infty,c]$, since $\phi_1'(x) = e^x / (1+e^x)$,
    \begin{align*}
        \frac {c_1 \phi_1(x) + c_2 \phi_2(x)}{c_1 \phi_1'(x) + c_2\phi_2'(x)}
        =
        \frac{c_1 \ln(1+e^x) + c_2 e^x}{c_1 e^x / (1+e^x) + c_2 e^x}
        &\leq
        \frac{e^x(c_1 + c_2)}{e^x(c_1 / (1+c) + c_2)},
    \end{align*}
    and so it suffices to set $\beta := (c_1 + c_2) / (c_1 / (1+c) + c_2)$.  Furthermore,
    since $\phi_1''(x) = e^x / (1+e^x)^2$,
    \begin{align*}
        \frac {c_1 \phi_1''(x) + c_2 \phi_2''(x)}{c_1 \phi_1(x) + c_2\phi_2(x)}
        =
        \frac{c_1 e^x / (1+e^x)^2 + c_2 e^x}{c_1 \ln(1+e^x) + c_2 e^x}
        &\leq
        \frac{e^x(c_1 + c_2)}{e^x(c_1 \ln(1+c) / c + c_2)},
    \end{align*}
    thus $\eta := (c_1 + c_2) / (c_1 \ln(1+c) / c + c_2)$ suffices.
\end{proof}

\section{Deferred material from \Cref{sec:impossibility}}
\label{sec:app:impossibility}
\begin{proof}[Proof of \Cref{fact:separable_case_impossibility}]
    As stated in the \namecref{fact:separable_case_impossibility},
    set $\cX = [-1,+1]^2$, and $\cH$ to be the two projection
    maps $h_1 (x) = x_1$ and $h_2(x) = x_2$.
    Next define a set of positive instances $\{p_i\}_{i=1}^\infty$, and their
    corresponding probability mass:
    \begin{align*}
        p_i = \left[\begin{smallmatrix}
                1- 0.5\cdot 4^{2-i} \\
                1
        \end{smallmatrix}\right],
        \qquad\mu(p_i) = 2^{-i-1}.
    \end{align*}
    Here are the negative instances:
    \begin{align*}
        n_i = \left[\begin{smallmatrix}
                1 \\
                1 - 0.3\cdot 4^{2-i}
        \end{smallmatrix}\right],
        \qquad\mu(n_i) = 2^{-i-1}.
    \end{align*}

    Notice that $\mu$ has countable support, and $\mu(\cX) = 1$.  Furthermore,
    the vector $\bar\lambda = (-1,+1)$ is a perfect separator:
    given any positive example $p_i$, $(H\bar\lambda)(p_i) > 0$, and given
    negative example $n_i$, $(H\bar\lambda)(n_i) < 0$.  Note however
    that, as required by the \namecref{fact:separable_case_impossibility} statement,
    the margins go to zero. However, given any
    $\phi\in \Phi$, since $\lim_{z\to-\infty}\phi(z) = 0$,
    \[
        0  \leq \inf_\lambda \cR_\phi(H\lambda)
        \leq \lim_{c\uparrow\infty} \int \phi(-y_i(Hc\bar\lambda)(z_i))d\mu(z_i,y_i)
        = 0.
    \]

    The key property
    of this construction is that the positive and negative examples are staggered; this
    will cause max margin solutions to avoid $\bar\lambda$.  As such, let any
    finite sample of size $m$ be given.  If all drawn examples have the same class $y$,
    then $\hat\lambda = (1-y,1+y)$ (which is a maximum margin solution)
    has either $n_1$ or $p_1$ on the wrong side of the
    separator, and by choosing
    $c>0$ large enough, $\cR_\phi(cH\hat\lambda) > b$.

    As such, henceforth suppose there is at least one positive example, and at least
    one negative example.
    Suppose $j$ and $k$ respectively denote a sampled positive
    point $p_j$ and sampled negative point $n_k$ having highest index among positive
    and negative examples; these maxima exist since $m$ is finite.

    Every max margin
    solution is determine solely by $p_j$ and $n_k$.  To obtain one of them,
    define
    \[
    \lambda
    :=
    \left[
        \begin{smallmatrix}
            -(1+ (n_k)_2) / (2 + (p_j)_1 + (n_k)_2) \\
            (1+ (p_j)_1) / (2 + (p_j)_1 + (n_k)_2)
        \end{smallmatrix}
    \right].
    \]
    To verify that this is a max margin solution, note that for any sampled
    (positive or negative) point $z_i$ with label $y_i\in\{-1,+1\}$,
    \[
        y_i (H\lambda) z_i \geq (H\lambda) (p_j) = -(H\lambda)(n_k)
        = -\ip{\lambda}{n_k} =
        \frac {(p_j)_1(n_k)_2}{2 + (p_j)_1 + (n_k)_2} > 0.
    \]
    By construction, however, $(p_j)_1 \neq (n_k)_2$, meaning $\lambda$ is not
    a rescaling of $\bar\lambda$.
    As such, $\lambda$ is wrong for either all large $p_i$ or $n_i$,
    and taking $\hat\lambda = q\lambda$ with $q$ large, it follows that
    $\cR_\phi(H\hat\lambda) > b$.
\end{proof}

\section{Deferred material from \Cref{sec:hard_cores}}
\label{sec:app:hard_cores}

Throughout this section, the following notation for measures will be employed
\begin{definition}
    Given a measure $\mu$ and a set $P$, let $\mu_P$ be the restriction of $\mu$ to
    $P$: for any measurable set $S$, $\mu_P(S) = \mu(P\cap S)$.  Note also
    that $d\mu_P(x,y) = \1((x,y)\in P) d\mu(x,y)$.
\end{definition}

\subsection{Proof of \Cref{fact:hard_core:existence}}
In order to establish the existence of hard cores, this section
first establishes a few properties of $\cD(\cH,\mu)$ and $\cS_\cD(\cH,\mu)$.

\begin{lemma}
    \label{fact:hard_core:countable_additive}
    Given any $\{c_i\}_{i=1}^\infty$ with $c_i\geq 0$ and
    $\{p_i\}_{i=1}^\infty$ with $p_i\in\cD(\cH,\mu)$ and
    $\sum_i c_i \|p_i\|_1 < \infty$, the limit object
    $p_\infty := \sum_i c_ip_i$ exists, and safisfies
    $p_\infty \in \cD(\cH,\mu)$.
\end{lemma}

\begin{proof}
    Let $\{c_i\}_{i=1}^\infty$ and $\{p_i\}_{i=1}^\infty$ be given
    as specified.
    First, by the monotone convergence theorem, the function $p_\infty = \sum_i c_ip_i$
    exists (i.e., all limits converge pointwise), is measurable,
    and safisfies $\int p_\infty = \sum_i \int c_ip_i < \infty$, meaning
    $p_\infty \in L^1(\mu)$ \citep[Theorem 2.15]{folland}.
    Now let any $\lambda\in\R^n$ be given; note that
    $\sum_i \int |c_ip_i (H\lambda)| \leq \|H\lambda\|_\infty\sum_i \|c_ip_i\|_1<\infty$.
    Thanks to this,
    by the dominated convergence theorem \citep[Theorem 2.25]{folland},
    \begin{align*}
        \int p_\infty(x,y) y(H\lambda)x d\mu(x,y)
        &= \int \sum_{i=1}^\infty c_ip_i(x,y)y(H\lambda)x d\mu(x,y)
        \\
        &= \sum_{i=1}^\infty\int  c_ip_i(x,y) y(H\lambda)x d\mu(x,y)
        \\
        &= \sum_{i=1}^\infty c_i\int p_i(x,y) y(H\lambda)x d\mu(x,y)
        \\
        &= 0.
        \qedhere
    \end{align*}
\end{proof}

\begin{lemma}
    \label{fact:hard_core:countable_union}
    $\cS_\cD(\cH,\mu)$ is closed under countable unions.
\end{lemma}
\begin{proof}
    Let any collection $\{C_i\}_{i=1}^\infty$ with $C_i\in\cS_\cD(\cH,\mu)$ and
    corresponding weighting $p_i\in\cD(\cH,\mu)$ be given.
    Define
    \[
        C := \bigcup_{i=1}^\infty C_i
        \qquad\textup{and}\qquad
        p := \sum_{i=1}^\infty \frac {p_i}{2^i \max\{1,\|p_i\|_1\}}.
    \]
    By \Cref{fact:hard_core:countable_additive}, $p$ exists and satisfies
    $p \in \cD(\cH,\mu)$.  Note further that $C = [p>0]$, and thus $C\in\cS_\cD(\cH,\mu)$.
\end{proof}

\begin{proof}[Proof of \Cref{fact:hard_core:existence}]
    Consider the optimization problem
    \[
        d := \sup \{\mu(C) : C \in \cS_\cD(\cH,\mu)\}.
    \]
    Since $\cS_\cD$ is nonempty (always contains $\emptyset$ corresponding to
    $p=0\in\cD(\cH,\mu)$)
    and $\mu(\cX\times\cY) < \infty$, the supremum is finite.
    Let $\{C_i\}_{i=1}^\infty$ be a maximizing sequence,
    and define $D_j := \cup_{i\leq j} C_i$ and
    $D := \cup_{j=1}^\infty D_j = \cup_{i=1}^\infty C_i$.
    By \Cref{fact:hard_core:countable_union}, $D_j\in\cS_\cD(\cH,\mu)$ for every $j$,
    and since $\mu(D_j)\geq \mu(C_j)$, it follows that $\{D_j\}_{j=1}^\infty$ must also
    be a maximizing sequence to the above supremum.
    Finally, since
    \Cref{fact:hard_core:countable_union} also grants $D\in \cS_\cD(\cH,\mu)$,
    then
    by continuity of measures from below \citep[Theorem 1.8(c)]{folland},
    \[
        \mu(D) = \lim_{j\to\infty} \mu(D_j)
        = d.
    \]
    Since $D\in\cS_\cD(\cH,\mu)$ attains the supremum, it is a dual hard core.
\end{proof}

\subsection{Primal hard cores}
In light of the duality relationship for $\cR_\phi$ (cf. \Cref{fact:duality}), the
definition for hard cores, provided in \Cref{sec:hard_cores}, is tied to the convex
dual to $\cR_\phi$.  Analogously, it is possibly to define a primal form of hard cores,
which will be lead to a proof of \Cref{fact:partial_ultragordan}.

\begin{definition}
    Define $\cS_\cP(\cH,\mu)$ to contain all sets $C$ for which
    there exists a sequence $\{\lambda_i\}_{i=1}^\infty$ satisfying the following
    properties.
    \begin{enumerate}
        \item Every $\lambda_i$ and $(x,y)\in C$ satisfies $y(H\lambda_i)x = 0$.
        \item For $\mu$-almost-every $(x,y)$ in $C^c$, $y(H\lambda_i)x \uparrow \infty$.
    \end{enumerate}
    A \emph{primal hard core} $\scrP$ is a minimal set within $\cS_\cP(\cH,\mu)$:
    \[
        \scrP\in\cS_\cP(\cH,\mu)
        \qquad
        \textup{and}
        \qquad
        \forall C\in\cS_\cP(\cH,\mu)\centerdot
        \mu(\scrP\setminus C) = 0 \land \mu(C\setminus \scrP) \geq 0.
        \qedhere
    \]
\end{definition}

\begin{lemma}
    \label{fact:primal_hard_core:countable_intersection}
    $\cS_\cP(\cH,\mu)$ is closed under countable intersections.
\end{lemma}
\begin{proof}
    To start, note that $\cS_\cP(\cH,\mu)$ is closed under finite intersections
    as follows.
    Let $\{C_i\}_{i=1}^p$ be given with corresponding sequences
    $\{\LL{i}{j}\}_{j=1}^\infty$.
    Define $C := \cap C_i$ and $\lambda_j := \sum_i \LL{i}{j}$.
    By construction, for every $(x,y)\in C$ and pair $(i,j)$,
    $y(H\LL{i}{j})x = 0$, and thus $y(H\lambda_j)x = 0$.
    Next, for each $C_i$, define $C'_i \subseteq C_i^c$ with $\mu(C'_i) = \mu(C_i^c)$
    so that, for every $(x,y)\in C'_i$,
    $y(H\LL{i}{j})x \uparrow \infty$.  Correspondingly, define $C' := \cup_i C'_i$, where
    $\mu(C') = \mu(C^c)$.  Now let any $(x,y)\in C'$ and any $B>0$ be given.
    For each $i$, there are two cases: either this is an area
    where $y(H\LL{i}{j})x\uparrow\infty$,
    or $y(H\LL{i}{j})x = 0$.
    In the first case, let $T_i$ denote an integer, as granted by
    $y(H\LL{i}{j})x\uparrow \infty$, so that for all $j \geq T_i$, $y(H\LL{i}{j})x > B$.
    For those $i$ where $(x,y)\not\in C'_i$ (but still $(x,y)\in C'$),
    due to the ruled out nullsets,
    $y(H\LL{i}{j})x = 0$, safely set $T_i = 0$.
    To finish, taking $T := \max_i T_i$, it follows that for every $j > T$,
    $y(H\lambda_j)x > B$, whereby it follows that $y(H\lambda_j)x \uparrow \infty$ over
    $C'$, and thus over $C^c$ $\mu$-a.e.

    Now let a countable family $\{D_i\}_{i=1}^\infty$ be given, and define
    $D=\cap_i D_i$.  Consider the optimization problem
    \[
        p:=\inf\left\{\int \exp(-y(H\lambda)x)d\mu_{D^c}(x,y) : \lambda\in\R^n,
            \forall (x,y)\in D \centerdot y(H\lambda)x = 0
        \right\}.
    \]
    Define $E_j := \cap_{i\leq j} D_i$, whereby $D := \cap_j E_j$.
    Since $\mu(\cX\times\cY)<\infty$, by continuity of measures
    from above \citep[Theorem 1.8(d)]{folland}, for any $\tau>0$
    there exists $E_k$ with $\mu(D) > \mu(E_k) - \tau$.
    Since it was shown above that $\cS_\cP(\cH,\mu)$ is closed under finite intersections,
    $E_k = \cap_{i\leq k} D_i \in \cS_\cP(\cH,\mu)$; consequently,
    let $\{\lambda_i\}_{i=1}^\infty$ to be a sequence of predictors certifying
    that $E_k\in\cS_\cP(\cH,\mu)$, as according to the definition.
    It follows that
    \[
        p
        \leq \lim_{i\to\infty} \int \exp(-y(H\lambda_i)x)\mu_{D^c}(x,y)
        = 0 + \int \exp(0)\mu_{E_k\setminus D}
        = \mu(E_k) - \mu(D) < \tau.
    \]
    Since $\tau$ was arbitrary, it follows that $p=0$.

    As such, for any $n\in \Z_{++}$, choose $\bar\lambda_n\in\R^n$ 
    with $y(H\lambda_n)x=0$ over $D$ satisfying
    \[
        \int \exp(-y(H\bar\lambda_n)x)d\mu_{D^c}(x,y) < 1/n^2.
    \]
    By Markov's inequality, it follows that
    \[
        \mu_{D^c}([\exp(-y(H\bar\lambda_n)x) \geq 1/n])
            \leq n \int \exp(-y(H\bar\lambda_n)x)\mu_{D^c}(x,y)< 1/n.
    \]
    As such, by definition,
    $\exp(-y(H\bar\lambda_n)x)$ converges in measure
    to the function $\1((x,y)\in D)$.  Consequently, there exists a subsequence
    $\lambda_i^*$ with $\exp(-y(H\lambda_i^*)x) \to \1(D)$ $\mu$-a.e.
    \citep[Theorem 2.30]{folland}.
    This is only possible if $y(H\lambda_i^*)x\uparrow \infty$ for
    $\mu$-a.e $(x,y)\in D^c$,
    and the result follows, with $\{\lambda_i^*\}_{i=1}^\infty$ as the certifying sequence
    for $D$, since every $y(H\lambda_i^*)x = 0$ for $(x,y)\in D$ by construction.
\end{proof}

\begin{theorem}
    Every linear classification problem $(\cH,\mu)$ has a primal hard core.
\end{theorem}
\begin{proof}
    Consider the optimization problem
    \[
        p := \inf \{\mu(C) : C \in \cS_\cP(\cH,\mu)\}.
    \]
    Since $\cS_\cP$ is nonempty (it always contains $\cX\times\cY$
    with certifying sequence $\lambda_i = 0$ for every $i$)
    and $\mu$ is a finite nonnegative measure,
    the infimum is finite.
    Let $\{C_i\}_{i=1}^\infty$ be a minimizing sequence,
    and define $D_j := \cap_{i\leq j} C_i$ and
    $D := \cap_{j=1}^\infty D_j = \cap_{i=1}^\infty C_i$.
    By \Cref{fact:primal_hard_core:countable_intersection},
    $D_j\in\cS_\cP(\cH,\mu)$ for every $j$,
    and since $\mu(D_j)\leq \mu(C_j)$, it follows that $\{D_j\}_{j=1}^\infty$ must also
    be a minimizing sequence to the above infimum.
    Finally, since $\mu$ is finite and
    \Cref{fact:primal_hard_core:countable_intersection} also grants
    $D\in \cS_\cP(\cH,\mu)$,
    then
    by continuity of measures from above \citep[Theorem 1.8(d)]{folland},
    \[
        \mu(D) = \lim_{j\to\infty} \mu(D_j)
        = p.
    \]
    Since $D\in\cS_\cP(\cH,\mu)$ attains the infimum, it is a primal hard core.
\end{proof}

With existence of primal hard cores out of the way, the next key is the equivalence
to (dual) hard cores.

\begin{theorem}
    \label{fact:hard_core:equivalence}
    Let a linear classification problem $(\cH,\mu)$ be given,
    along with a hard core $\scrC$, as well as a primal hard core $\scrP$.
    Then $\scrC$ and $\scrP$ agree on all but a null set.
\end{theorem}

The proof needs the following \namecref{fact:hard_core_lemma:1}.

\begin{lemma}
    \label{fact:hard_core_lemma:1}
    Let a linear classification problem $(\cH,\mu)$,
    $C_1\in\cS_\cP(\cH,\mu)$,
    as well as a
    $\lambda_2\in\R^n$ be given,
    with $y(H\lambda_2)x \geq 0$ for
    $(x,y)\in C_1$ (but potentially $y(H\lambda_2)x < 0$ elsewhere).
    Then $C_1 \setminus[y(H\lambda_2)x > 0] \in\cS_\cP(\cH,\mu)$.
\end{lemma}
\begin{proof}
    Let $C_1,\lambda_2$ be given as specified.
    Let $\{\LL{1}{i}\}_{i=1}^\infty$ be a certifying sequence for
    $C_1$.
    Define $P := [y(H\lambda_2)x > 0]$ and
    $C_3 := C_1 \setminus P = C_1 \setminus [y(H\lambda_2)x > 0]$.

    Now let $i\in\Z_{++}$ be arbitrary; the following steps will construct
    $\LL{4}{i}$, a certifying sequence for $C_3$, meaning $C_3 \in \cS_\cP(\cH,\mu)$.


    First, let $c$ be sufficiently large
    so that $\LL{2}{i}:= c\lambda_2$ satisfies
    \[
        \int \exp(-y(H\LL{2}{i})x)\mu_P(x,y) < 1/i^2.
    \]
    By Markov's inequality, it follows that
    \begin{equation}
        \mu_P([\exp(-y(H\bar\lambda_2)x) \geq 1/i]) \leq i
            \int\exp(-y(H\bar\lambda_2)x)\mu_P(x,y) < 1/i.
        \label{eq:hard_core_lemma:markov}
    \end{equation}
    Consequently define $P_i := [y(H\bar\lambda_2)x > \ln(i)]$,
    where the above statements show $\mu(P_i) > \mu(P) - 1/i$.

    Next,
    since $\exp(-y(H\LL{1}{i})x)\to \1(C_1)$ $\mu$-a.e. and $\mu(\cX\times\cY)<\infty$,
    by Egoroff's Theorem \citep[Theorem 2.33]{folland}, this convergence is uniform
    over a subset $S_i$ with $\mu(S_i) > \mu(\cX,\cY)-1/i$.
    In particular, there exists an integer $T_i$ so that, for any $(x,y)\in S_i\cap C_1$,
    \[
        y(H\LL{1}{T_i})x > \|\LL{2}{i}\|_1 + \ln(i).
    \]

    As such, define $\LL{3}{i} := \LL{1}{T_i} + \LL{2}{i}$.
    First, for any $(x,y)\in C_3$ and any $i$,
    \[
    y(H\LL{3}{i})x = 0 = y(H\LL{1}{i})x = y(H\lambda_2)x.
    \]
    On the other hand, for any $(x,y)\in S_i\cap C_1$,
    \begin{align*}
        y(H\LL{3}{i})x
        &=
        y(H\LL{1}{T_i})x + y(H\LL{2}{i})x
        \\
        &> \|\LL{2}{i}\|_1 + \ln(i) - \|\LL{2}{i}\|_1
        = \ln(i).
    \end{align*}
    Lastly, as shown above, for any $(x,y)\in P_i$,
    \[
        y(H\LL{3}{i})x
        = 0 + y(H\LL{2}{i})x
        \geq \ln(i).
    \]
    Combining the above facts,
    \[
        \mu([|\exp(-y(H\LL{3}{i}) x) - \1[(x,y)\in C_3]| \geq 1/i]) <
        \mu(C_1^c\setminus S_i) + \mu(P\setminus P_i) \leq 2/i.
    \]
    It follows that $\exp(-y(H\LL{3}{i})x) \to \1((x,y)\in C_3)$
    in measure,
    and thus there is a subsequence $\{\LL{4}{i}\}_{i=1}^\infty$ which converges to
    $\1((x,y)\in C_3)$
    $\mu$-a.e. \citep[Theorem 2.30]{folland}.
    It follows that $\{\LL{4}{i}\}_{i=1}^\infty$ is the desired sequence
    certifying that $C_3\in\cS_\cP(\cH,\mu)$.
\end{proof}

\begin{proof}[Proof of \Cref{fact:hard_core:equivalence}]
    If $\mu(\scrP\setminus \scrC) > 0$, then by the maximality of $\scrC$,
    $\scrP$ is a set of positive measure away from any element of $\cS_\cD(\cH,\mu)$,
    an in particular $\scrP\not\in\cS_\cD(\cH,\mu)$,
    and thus
    \Cref{fact:somewhat_stiemke} provides the existence of
    $\lambda\in\R^n$ with $\mu(\scrP \cap [y(H\lambda)x\geq 0]) = \mu(\scrP)$
    and $\mu(\scrP\cap [y(H\lambda)x > 0]) > 0$.
    But then, by \Cref{fact:hard_core_lemma:1},
    $\scrP$ can be reduced into a smaller element of $\cS_\cP(\cH,\mu)$,
    contradicting its minimality.


    Now suppose $\mu(\scrC\setminus \scrP) > 0$,
    and set $\nu$ to to be the restriction of $\mu$ to $\scrC$: for any $C$,
    $\nu(C) := \mu(\scrC\cap C)$.
    Consider the optimization problem
    \[
        \inf\left\{
            \int \exp(-y(H\lambda)(x)) d\nu(x,y)
            :\lambda\in\R^n
        \right\}.
    \]
    Consider the sublevel set of 1-suboptimal points for this problem.
    By \Cref{fact:somewhat_stiemke}, there exists $B$ so that each $\lambda$ in this
    sublevel set has $\lambda'$ with $H\lambda = H\lambda'$ $\mu$-a.e. and
    $\|\lambda'\|_1 \leq B$.  However, by the definition of $\scrP$, there exists
    a sequence $\{\lambda_i\}_{i=1}^\infty$ which is zero over $\scrP$ and approaches
    $\infty$ $\mu-a.e.$ over $\scrP^c$, and in particular over the positive measure
    set $\scrC\setminus \scrP$.  Thus, taking any $\lambda$ in the 1-suboptimal set,
    notice that
    \[
        \lim_{i\to\infty} \int \exp(-y(H(\lambda+\lambda_i))x)d\nu(x,y)
        =
        \int \exp(-y(H\lambda)(x))\1((x,y)\not\in\scrP)d\nu(x,y)
        =: p.
    \]
    Since $\lambda$ has a bounded representation, $\exp(-y(H\lambda)x)\neq 0$,
    and thus $p < \cR_\phi(H\lambda)$
    \citep[Theorem 2.23(b)]{folland}.
    But since the objective function is continuous in $\lambda$ (cf. \Cref{fact:intphi_basic}),
    there must exist a large $j$ so that $\cR_\phi(H(\lambda + \lambda_j)) < \cR_\phi(H\lambda)$,
    and moreover $y(H(\lambda+\lambda_j))(x) > B$ for a subset of $\scrC$ with positive measure.
    But that means $\lambda+\lambda_j$ is in the 1-sublevel set, but can not have a representation
    with norm at most $B$ (since $H$ is a bounded linear operator),
    contradicting \Cref{fact:somewhat_stiemke}.
%
\end{proof}

\subsection{Proof of \Cref{fact:partial_ultragordan}}
This is now just a consequence of the equivalence to primal hard cores, and
the structure over $\scrC$ developed in \Cref{fact:somewhat_stiemke} (which was
used to prove the equivalence to primal hard cores as well).

\begin{proof}[Proof of \Cref{fact:partial_ultragordan}]
    The second property is direct from \Cref{fact:somewhat_stiemke}.
    For the first property, since primal hard cores exist and are $\mu$-a.e. equivalent
    to hard cores (cf. \Cref{fact:hard_core:equivalence}),
    and statement thus follows by taking the sequence provided by
    the definition of any primal hard core.
\end{proof}

\section{Deferred material from \Cref{sec:hard_core:true_risk}}
\begin{proof}[Proof of \Cref{fact:hard_core:true_risk}]
    (\Cref{fact:hard_core:true_risk:1})  Let $\{\lambda_i\}_{i=1}^\infty$ be
    given as per \Cref{fact:partial_ultragordan}.
    Automatically, $y(H\lambda_i)x= 0$ for $(x,y)\in\scrC$.
    And since $y'(H\lambda_i)x' \uparrow \infty$ for $\mu$-a.e. $(x',y')\in \scrC^c$,
    it follows from the definition of $\Phi$ that
    $\lim_{i\to\infty} \phi(-y'(H\lambda_i)x) = 0$.

    (\Cref{fact:hard_core:true_risk:2})  This is a consequence
    of \Cref{fact:somewhat_stiemke}.
\end{proof}

\begin{proof}[Proof of \Cref{fact:hard_core:empirical_risk}]
    (\Cref{fact:hard_core:empirical_risk:1})
    Let a sequence $\{\lambda_i\}_{i=1}^\infty$ be given as provided by
    \Cref{fact:partial_ultragordan}.  In particular,
    $\exp(-y(H\lambda_i)x) \to \1(\scrC)$ $\mu$-a.e.
    Now choose a finite sample size $m$;
    by Egoroff's Theorem
    \citep[Theorem 2.33]{folland},
    for any $\tau>0$, there exists $S_\tau$ with
    $\mu(S_\tau) > \mu(\cX\times\cY) - \tau/m$ over which
    this convergence is uniform.  As such, choose $\lambda_\tau$
    so that $\exp(-y(H\lambda_\tau)x) < 1/2$
    over $S_\tau\cap \scrC^c$, meaning in particular
    $y(H\lambda_\tau)x > 0$ for every $(x,y)\in S_\tau \cap \scrC^c$.
    The probability over a draw of $m$ points that some
    within $\scrC^c$ are misclassified
    by $\lambda_\tau$ has upper bound bound
    \begin{align*}
        \mu^m(\exists (x_i,y_i)\in\scrC^c\centerdot y(H\lambda_i)x \leq 0)
        &\leq m \mu(\scrC^c \cap [y(H\lambda_i)x \leq 0])
        < \tau.
    \end{align*}
    Since $\tau$ can be made arbitrarily small, the probability of failure is zero.
    Furthermore, since  $\lambda_\tau$ satisfies $y(H\lambda_\tau)(x) = 0$ 
    $\mu$-a.e. over
    $\scrC$ (cf. \Cref{fact:partial_ultragordan}), it also
    follows that, with probability 1, $\lambda_\tau$ abstains on every
    example falling within $\scrC$.

    (\Cref{fact:hard_core:empirical_risk:2})
    Let $\rho>0$ and $\phi\in\Phi$ be given.
    Choose $b>0$, as provided by \Cref{fact:hard_core:true_risk},
    so that every $\lambda\in\R^n$ with
    $\cR_{\phi;\scrC}(H\lambda) \leq \cR_{\phi;\scrC}(\SPAN(\cH))+4+\rho$
    has a representation $\lambda'$ with $\|\lambda'|_1\leq b$,
    where $H\lambda = H\lambda'$ everywhere along $\scrC\setminus N$, where
    $\mu(N) = 0$;
    henceforth, rule out the event that any example falls within $N$.
    Additionally, choose $c>0$ as provided by \Cref{fact:rademacher_risk} so
    that, given $m_\scrC$ i.i.d. points within $\scrC\setminus N$,
    every $f\in\SPAN(\cH,b)$ has
    \begin{equation}
        |\cR_{\phi;\scrC}(f) - \cR_{\phi;\scrC}^m(f)|
        \leq c \frac{\sqrt{\ln(n)} + \sqrt{\ln(2/\delta)}}{\sqrt{m_\scrC}}.
        \label{eq:hard_core:empirical_risk:1}
    \end{equation}

    Now consider any $\lambda\in\R^n$ with
    no representation $\|\lambda'\|_1 \leq b$
    so that $H\lambda = H\lambda'$ over $\scrC\setminus N$,
    which directly entails, by \Cref{fact:hard_core:true_risk},
    that
    $\cR_{\phi;\scrC}(H\lambda) - \cR_{\phi;\scrC}(\SPAN(\cH)) >\rho + 4$.
    Additionally choose
    and any $\bar\lambda\in\R^n$ with
    $\cR_{\phi;\scrC}(H\bar\lambda)-\cR_{\phi;\scrC}(\SPAN(\cH)) < 1$,
    whereby the choice of $b>0$ indicates that, without loss of generality,
    $\|\bar\lambda\|_1\leq b$.
    Since $\int \phi\circ H$ is continuous (cf. \Cref{fact:intphi_H_basic}),
    considering
    the line segment $\{\alpha \lambda + (1-\alpha)\bar\lambda:\alpha\in[0,1]\}$,
    there must exist $\hat\lambda$ with
    \[
        \rho + 3 \leq \cR_{\phi;\scrC}(H\hat\lambda)
        - \cR_{\phi;\scrC}(\SPAN(\cH)) \leq \rho + 4;
    \]
    let $\hat\lambda'$ be a representation with $\|\hat\lambda'\|_1\leq b$
    and $H\hat\lambda = H\hat\lambda'$ over $\scrC\setminus N$ (and thus it
    holds for every example).
    Applying the deviation inequality in
        \cref{eq:hard_core:empirical_risk:1} twice,
    \begin{align*}
        \cR_{\phi;\scrC}^m(H\hat\lambda) - \cR_{\phi;\scrC}^m(H\bar\lambda)
        &\geq \cR_{\phi;\scrC}(H\hat\lambda') - \cR_{\phi;\scrC}(H\bar\lambda)
        -2 c \frac{\sqrt{\ln(n)} + \sqrt{\ln(2/\delta)}}{\sqrt{m_\scrC}}.
        \\
        &= \cR_{\phi;\scrC}(H\hat\lambda') -\cR_{\phi;\scrC}(\SPAN(\cH))
        - (\cR_{\phi;\scrC}(H\bar\lambda)-\cR_{\phi;\scrC}(\SPAN(\cH)))
        \\
        &\qquad-2 c \frac{\sqrt{\ln(n)} + \sqrt{\ln(2/\delta)}}{\sqrt{m_\scrC}}.
        \\
        &> (\rho + 3) - (1)
        -2 c \frac{\sqrt{\ln(n)} + \sqrt{\ln(2/\delta)}}{\sqrt{m_\scrC}}.
        \\
        &\geq \rho,
    \end{align*}
    where the last step used the lower bound on $m_\scrC$.
    Returning to $\lambda\in\R^n$ as specified above,
    convexity, in the form of \Cref{fact:convex:increasing},
    grants that $\cR_{\phi;\scrC}^m(H\bar\lambda)< \cR_{\phi;\scrC}^m(H\hat\lambda)$ implies
    $\cR^m_{\phi;\scrC}(H\hat\lambda) \leq \cR^m_{\phi;\scrC}(H\lambda)$, and thus
    \[
        \cR_{\phi;\scrC}^m(H\lambda) - \cR_{\phi;\scrC}^m(\SPAN(\cH))
        \geq \cR_{\phi;\scrC}^m(H\hat\lambda) - \cR_{\phi;\scrC}^m(H\bar\lambda) > \rho.
    \]
    Since $\lambda$ was arbitrary, it follows that every $\lambda$ with
    no representation
    $\|\lambda'\|_1  > b$ that has agreement of $H\lambda$ and $H\lambda'$
    $\mu$-a.e. over $\scrC$ does not lie in the empirical $\rho$-sublevel set.
    Since $\cR_{\phi;\scrC}^m$ is convex and continuous, the $\rho$-sublevel set is
    nonempty, and thus every $\lambda'$ within it has a representation
    $\|\lambda''\|_1\leq b$.
\end{proof}

\section{Deferred material from \Cref{sec:deviations}}
\label{sec:app:deviations}

\begin{proof}[Proof of \Cref{fact:psi:basics}]
    This proof is essentially a repackaging of various results and comments
    due to
    \citet{bartlett_jordan_mcauliffe}.  Fix any $\phi\in\Phi$;
    $\phi$ is convex, increasing at 0, and differentiable at 0,
    which grants that the corresponding $\psi$-transform
    is classification calibrated \citep[Theorem 6, although note losses in the
    present manuscript are increasing rather than decreasing]{bartlett_jordan_mcauliffe}.
    It follows that
    $
        \psi(\cR_\cL(f) -\cR_\cL(\fF)) \leq \cR_\phi(f) - \cR_\phi(\fF),
    $
    \citep[Theorem 3, part 3(c)]{bartlett_jordan_mcauliffe}.

    Next, $\psi(0)=0$
    \citep[Lemma 5, part 8]{bartlett_jordan_mcauliffe},
    $\psi(r)>0$ when $r>0$
    \citep[Lemma 5, part 9(b)]{bartlett_jordan_mcauliffe},
    and since $\psi$ is convex by construction
    \citep[Definition 2]{bartlett_jordan_mcauliffe},
    it follows by \Cref{fact:convex:increasing}
    that $\psi$ is increasing. Since $\psi$ is continuous as well,
    \citep[Lemma 5, part 6]{bartlett_jordan_mcauliffe},
    it follows that $\psi$ has a well-defined inverse along the image
    $\psi([0,1])$.
    Finally, the fact that $\psi^{-1}(r)\downarrow 0$ as $r\downarrow 0$
    is due to \citet[Theorem 3, part 3(b)]{bartlett_jordan_mcauliffe}.
\end{proof}

\begin{proof}[Proof of \Cref{fact:deviations}]
    Throughout this proof, $\delta' := \delta/8$ will be the failure probability of
    various crucial events; the final statement is obtained by unioning them together,
    and subsequently throwing them all out.  Note also that some of the statements
    vacuously hold if $\mu(\scrC) = 0$ or $\mu(\scrC) = \mu(\cX\times \cY)$ (i.e., when
    terms depending on either appear in denominators); interpret these expressions
    as simply being $\infty$, whereby the bounds hold automatically.

    (\Cref{fact:deviations:1})
    Let $\cS_\scrC$ and $\cS_+$ respectively denote the set of samples landing
    in $\scrC$ and $\scrC^c$, where the notation proposed in the \namecref{fact:deviations}
    statement provides $m_\scrC = |\cS_\scrC|$ and $m_+ = |\cS_+|$.
    By a 
    Chernoff bound \citep[Theorem 9.2]{kearns_vazirani}, basic deviations for these
    quantities are
    \begin{align*}
        \textup{Pr}^m[ |\cS_\scrC| < (\mu(\scrC)-\tau)m]
        &\leq \exp(-m\tau^2/2) = \delta',
        \\
        \textup{Pr}^m[ |\cS_+| < (\mu(\scrC^c)-\tau) m]
        &\leq \exp(-m\tau^2)/2 = \delta',
    \end{align*}
    where $\tau = \sqrt{\frac 1 {2m} \ln \left(\frac 1 {\delta'}\right)}$,
    and $\textup{Pr}^m$ denotes the product measure corresponding to $\mu$.
    Label these
    failure events $F_1$ and $F_2$, and henceforth rule them out.

    (\Cref{fact:deviations:2})
    As provided by \Cref{fact:hard_core:empirical_risk},
    there exists $\bar\lambda\in\R^n$ with $y_i (H\bar\lambda)x_i > 0$ for all
    $(x_i,y_i)$ falling in $\scrC^c$, and $y_i(H\bar\lambda)x_i = 0$ for those
    landing in $\scrC$.  Consequently,
    \begin{align*}
        \cR_\phi(\SPAN(\cH))
        &=
        \inf_\lambda\inf_{c>0} \cR_{\phi,\scrC}(H(\lambda + c\bar\lambda))
        +  \cR_{\phi,\scrC^c}(H(\lambda + c\bar\lambda))
        \\
        &=
        \inf_\lambda\inf_{c>0} \cR_{\phi,\scrC}(H\lambda)
        \\
        &\leq
        \cR_\phi(\SPAN(\cH)).
    \end{align*}
    Combining this with
    \[
        \cR_{\phi,\scrC^c}(H\lambda)
        +
        \cR_{\phi,\scrC}(H\lambda) = \cR_{\phi}(H\lambda)
        \leq \cR_\phi(\SPAN(\cH)) + \epsilon,
    \]
    it follows that
    \[
        \cR_{\phi,\scrC^c}(H\lambda)
        \leq \cR_\phi(\SPAN(\cH)) - \cR_{\phi,\scrC}(H\lambda) + \epsilon
        = \cR_{\phi,\scrC}(\SPAN(\cH)) - \cR_{\phi,\scrC}(H\lambda) + \epsilon
        \leq \epsilon.
    \]
    Next, since $\phi(0) > 0$ and $\phi$ is nondecreasing (cf. \Cref{fact:phi_basic}),
    \[
        \cR_{\cL,\scrC^c}^m(H\lambda)
        \leq \frac {\cR_{\phi,\scrC^c}^m(H\lambda)}{\phi(0)}
        = \frac {\epsilon} {\phi(0)}.
    \]
    To obtain \cref{eq:deviations:2:1} from here, first notice that $\cS_+$,
    the portion of
    the sample falling within $\scrC^c$, can be interpreted as an i.i.d.
    sample
    from the probability measure $\mu(\cdot \cap \scrC) / \mu(\scrC)$.
    Next,
    the VC dimension of $\SPAN(\cH)$ is the VC dimension of linear
    separators over the transformed space
    \[
        \left\{
            \left(
            (h_1(x),h_2(x),\ldots,h_n(x))
            ,y
            \right)
            :
            (x,y)\in\cX\times \cY
        \right\};
    \]
    namely, it is $n$.
    As such, \cref{eq:deviations:2:1} follows by an application of a
    relative deviation version of the VC Theorem
    \citep[discussion preceding Corollary 5.2]{bbl_esaim}.

    To obtain \cref{eq:deviations:2:2}, note that $\epsilon < \phi(0)/m$
    means there are no mistakes over $\scrC^c$:
    \begin{align*}
        \phi(0) > m\epsilon
        &\geq m\left(\cR_\phi^m(H\hat\lambda) - \bar\cR_\phi^m(\SPAN(\cH))\right)
        \\
        &\geq
        m_+\cR_{\phi;\scrC^c}^m
        \\
        &\geq
        \sum_{i=1}^{m_+} \phi(-y_i (H\hat\lambda)x_i)
        \\
        &\geq \max_{i\in[m_+]} \phi(-y_i (H\hat\lambda)x_i);
    \end{align*}
    that is to say, for every $(x_i,y_i) \in \cS_+$, $0 < y_i (H\hat\lambda)x_i$.
    Plugging $\cR_\cL^m(H\lambda) = 0$ into the same relative deviation bound as
    before \citep[discussion preceding Corollary 5.2]{bbl_esaim}, the second bound
    follows.

    (\Cref{fact:deviations:3})
    By \Cref{fact:hard_core:empirical_risk},
    there exist constants $b>0$ and $c\geq \phi(b)$, depending on $\cH,\mu,\phi,\scrC$,
    so that with probability at least $1-\delta'$,
    if $m_\scrC \geq c^2(\ln(n) + \ln(1/\delta'))$,
    then every $\rho$-suboptimal predictor over $\scrC$, and in particular
    $\lambda$, has a representation $\lambda'$ which is equivalent to $\lambda$
    $\mu$-a.e. over $\scrC$, and satisfies $\|\lambda'\|_1\leq b$.
    As such, since
    \[
        \cR_{\phi;\scrC}^m(H\lambda)  =
        \cR_{\phi;\scrC}^m(H\lambda')
        \qquad
        \textup{and}
        \qquad
        \cR_{\phi;\scrC}(H\lambda)  =
        \cR_{\phi;\scrC}(H\lambda'),
    \]
    an application of \Cref{fact:rademacher_risk} grants
    \begin{align*}
        \cR_{\phi;\scrC}(H\lambda)
        &=
        \cR_{\phi;\scrC}(H\lambda')
        \\
        &\leq
        \cR_{\phi;\scrC}^m(H\lambda')
        + \frac {c\left(\sqrt{\ln(n)} + \sqrt{\ln(2/\delta')}\right)}{\sqrt {m_{\scrC}}}
        \\
        &=
        \cR_{\phi;\scrC}^m(H\lambda)
        + \frac {c\left(\sqrt{\ln(n)} + \sqrt{\ln(2/\delta')}\right)}{\sqrt {m_{\scrC}}}
        \\
        &\leq
        \cR_{\phi;\scrC}^m(\SPAN(\cH)) + \epsilon
        + \frac {c\left(\sqrt{\ln(n)} + \sqrt{\ln(2/\delta')}\right)}{\sqrt {m_{\scrC}}}.
    \end{align*}
    Next, noting that \Cref{fact:hard_core:true_risk} provides that a minimizing
    sequence to $\cR_{\phi;\scrC}(\SPAN(\cH))$ can be taken without loss of generality
    to lie within a compact set (e.g., points with $l^1$ norm at most $b$),
    it follows that a minimizer $\bar\lambda$ exists; by an application of McDiarmid's
    inequality, with probability at least $1-\delta'$,
    \[
        \cR_{\phi;\scrC}^m(\SPAN(\cH))
        \leq \cR_{\phi;\scrC}^m(H\bar\lambda)
        \leq
        \cR_{\phi;\scrC}(H\bar\lambda) + c\sqrt{\frac{2\ln(1/\delta')}{m_\scrC}}.
    \]
    (Note, $\bar\lambda$ is independent of the sample, thus McDiarmid
    suffices, with constant $c\geq \phi(b)$ since $\bar\lambda$ is in this
    initial sublevel set.)
    Combining these two pieces, it follows that
    \begin{align*}
        \cR_{\phi;\scrC}(H\lambda)
        - \cR_{\phi;\scrC}(\SPAN(\cH))
        &\leq
        \epsilon
        + \frac {c\left(\sqrt{\ln(n)} + 4\sqrt{\ln(2/\delta')}\right)}{\sqrt {m_{\scrC}}}
        ,
    \end{align*}
    which is precisely \cref{eq:deviations:3:1}.

    To produce \cref{eq:deviations:3:2},
    the definition of the $\psi$-transform (cf. \Cref{fact:psi:basics}),
    combined with \Cref{eq:deviations:3:1}, provides
    \begin{align*}
        \cR_{\cL;\scrC}(H\lambda) - \cR_{\cL;\scrC}(\fF)
        &\leq \psi^{-1}\left(\cR_{\phi;\scrC}(H\lambda)
        - \cR_{\phi;\scrC}(\fF)
        \right)
        \\
        &= \psi^{-1}\left(\cR_{\phi;\scrC}(H\lambda)
        - \cR_{\phi;\scrC}(\SPAN(\cH))
        +\cR_{\phi;\scrC}(\SPAN(\cH))
        - \cR_{\phi;\scrC}(\fF)
        \right)
        \\
        &\leq \psi^{-1}\left(
        \epsilon
        + \frac {c\left(\sqrt{\ln(n)} + 4\sqrt{\ln(2/\delta')}\right)}{\sqrt {m_{\scrC}}}
        +
        \cR_{\phi;\scrC}(\SPAN(\cH))
        - \cR_{\phi;\scrC}(\fF)
        \right).
    \end{align*}

    (\Cref{fact:deviations:4})
    Combining the lower bound on $m$ with \Cref{fact:deviations:1},
    \begin{align*}
        m_+ &\geq m \mu(\scrC^c) / 2,\\
        m_\scrC &\geq m \mu(\scrC) / 2
        \geq c^2(\ln(n) + \ln(1/\delta'));
    \end{align*}
    the first two bounds will allow expressions to be simplified, whereas the
    last bound will allow an invocation of \cref{fact:deviations:3}.

    As such, combining all preceding bounds (and making use of the
    refinement over $\scrC^c$ when $\epsilon<\phi(0)/m$),
    \begin{align*}
        \cR_\cL(H\lambda) - \cR_\cL(\fF)
        &= (\cR_{\cL;\scrC}(H\lambda) - \cR_{\cL;\scrC}(\fF))
        + (\cR_{\cL;\scrC^c}(H\lambda) - \underbrace{\cR_{\cL;\scrC^c}(\fF)}_{=0})
        \\
        &\leq
        \psi^{-1}\left(
        \epsilon
        + \frac {c\left(\sqrt{\ln(n)} + 4\sqrt{\ln(2/\delta')}\right)}{\sqrt {m_{\scrC}}}
        +\cR_{\phi;\scrC}(\SPAN(\cH))
        - \cR_{\phi;\scrC}(\fF)
        \right)
        \\
        &\qquad
        + \frac{4(n\ln(2m_++1) + \ln(4/\delta')} {m_+}
        \\
        &\leq
        \psi^{-1}\left(
        \epsilon
        + \frac {c\sqrt{2}\left(\sqrt{\ln(n)} + 4\sqrt{\ln(2/\delta')}\right)}
        {\sqrt {m\mu(\scrC)}}
        +\cR_{\phi;\scrC}(\SPAN(\cH))
        - \cR_{\phi;\scrC}(\fF)
        \right)
        \\
        &\qquad
        + \frac{8(n\ln(m\mu(\scrC^c)+1) + \ln(4/\delta')} {m\mu(\scrC^c)}.
    \end{align*}
\end{proof}

\section{Deferred material from \Cref{sec:consistency}}
\label{sec:app:consistency}

\begin{proof}[Proof of \Cref{fact:consistency:cheating}]
    Let $\scrC$ be a hard core for $(\cH,\mu)$,
    set $\rho := 1$, and let $b>0$ and $c>0$ be the corresponding reals
    provided in the guarantee of \Cref{fact:deviations}.
    Note first that $\cR_\phi(\SPAN(\cH)) = \cR_\phi(\fF)$ implies
    $\cR_{\phi;\scrC}(\SPAN(\cH)) = \cR_{\phi;\scrC}(\fF)$, since
    predictions are $\mu$-a.e. perfect off the hard core (cf.
    \Cref{fact:hard_core:true_risk}).
    Set $\delta_i = 1/i^2$, and choose $m_i\uparrow\infty$ large enough
    and $\epsilon_i\downarrow 0$ small enough so
    that the relevant finite sample bound from \Cref{fact:deviations} holds,
    and goes to zero.  (Note that all bounds go to zero as $m_i\uparrow \infty$ and
    $\epsilon_i\downarrow 0$; the word ``relevant'' refers to choosing
    a bound corresponding to the regime $\mu(\scrC)=0$, or $\mu(\scrC^c)=0$,
    or $\min\{\mu(\scrC),\mu(\scrC^c)\} >0$.)  Note, by the strong assumption,
    the term $\cR_{\phi;\scrC}(\SPAN(\cH)) - \cR_{\phi;\scrC}(\fF)$ may be dropped.

    Now let $F_i$ be the failure event of the corresponding finite sample guarantee;
    by choice of $\delta_i$,
    $\sum_i \textup{Pr}(F_i) = \sum_i i^{-2} = \pi^2/6 < \infty$.
    Thus, by the Borel-Cantelli Lemma and de Morgan's Laws,
    $\textup{Pr}(\liminf_{i\to\infty} F_i^c) = 1$,
    meaning $\textup{Pr}(\exists j\centerdot \forall i\geq
    j \centerdot F_i^c)=1$.
    This means that the bounds hold for all large $i$ (with probability 1),
    and the result follows by choice of $m_i$ and $\epsilon_i$.
\end{proof}

\begin{proof}[Proof of \Cref{fact:dt:lsrmf}]
    This proof will proceed in the following stages.  First, it is shown that
    the infimal risk $\cR_\phi(\fF)$ can be approximated arbitrarily well by
    bounded measurable functions.  Next, Lusin's theorem will allow this consideration
    to be restricted to a function which is continuous over a compact set.
    Finally, this function is approximated by a decision tree.

    Let $\mu$, $\phi$, $\{\cH_i\}_{i=1}^\infty$, and $\epsilon > 0$ be given
    as specified.  Since the infimum in $\cR_\phi(\fF)$ is in
    general not attained, let $g\in\fF$ be a measurable function satisfying
    \[
        \cR_\phi(g) \leq \epsilon / 4 + \cR_\phi(\fF).
    \]

    Next let $z >0$ be a sufficiently large real so that $\phi(-z) < \epsilon/4$;
    such a value must exist since $\lim_{z\to-\infty}\phi(z) = 0$.
    Correspondingly, define a truncation of $g$ as
    \[
        \hat g(x) := \min\{ z , \max\{ -z, g(x)\}\}.
    \]
    There are three cases to consider.  If $|y g(x)| \leq z$,
    then $\phi(-y\hat g(x)) = \phi(-yg(x))$.  If $-yg(x) > z$, then by the nondecreasing
    property (cf. \Cref{fact:phi_basic}), $\phi(-yg(x)) \geq \phi(-y\hat g(x))$.
    Lastly, if $-yg(x) < -z$, then $\phi(-y\hat g(x)) \leq \phi(-y g(x)) + \epsilon/4$
    by choice of $z$.  Together, it follows that
    \begin{align*}
        \cR_\phi(\hat g)
        &= \int \phi(-y\hat g(x))d\mu(x,y)
        \\
        &\leq \int (\phi(-y g(x)) + \epsilon/4)d\mu(x,y)
        \\
        &= \cR_\phi(g) + \epsilon\mu(\cX\times\cY) / 4
        \\
        &\leq \cR_\phi(\fF) + \epsilon/2,
    \end{align*}
    which used the fact that $\mu$ is a probability measure.
    Crucially, $\hat g$ is now a bounded measurable function.
    Throughout the remained of this proof, let $\|\cdot\|_u$
    denote the uniform norm, meaning
    \[
        \|f\|_u := \sup_{x} |f(x)|.
    \]
    For example, $\|\hat g\|_u < \infty$.

    In order to apply Lusin's Theorem and pass to continuous functions with compact
    support, a few properties must be verified.
    First, since $\mu_\cX$ is a Borel probability measure, it is finite on all compact
    Borel sets.
    Next, $\R^d$ is a separable metric space,
    and thus second countable.
    Finally, $\R^d$ is a locally compact Hausdorff space.
    It follows that $\mu_\cX$ is a Radon measure \citep[Theorem 7.8]{folland}.

    Henceforth, set $\tau := \epsilon / (8 \max\{1,\phi(\|\hat g\|_u)\})$.
    By Lusin's Theorem, there exists a measurable function $h$ which
    is
    continuous, has compact support, satisfies $\mu_\cX([\hat g \neq h]) < \tau$
    and $\|h\|_u \leq \|\hat g\|_u$ \citep[Theorem 7.10, Lusin's Theorem]{folland}.
    But continuity over a compact set implies uniform continuity.
    Furthermore, the convex function $\phi$, restricted to the domain $[-z,z]$,
    is necessarily Lipschitz.
    As such, it is possible to choose $\delta > 0$
    so that for any $x,x'$ with $\|x-x'\|_\infty < \delta$
    and any $y\in\{-1,+1\}$,
    it follows that $|\phi(-yh(x)) - \phi(-yh(x'))| < \tau$.  Notice that this in
    fact holds everywhere, since outside of its support $h$ is just zero.

    As such, let $T$ be the smallest integer so that
    $T> \sup\{\|x\|_\infty : h(x) \neq 0\}$
    (which exists since $h$ has compact support)
    and also $1/T < \delta$.
    For any $t\geq T$, construct a simple function approximation $f$ to $h$ as follows.
    Partition the cube $[-t,t)^d$ into subcubes (formed as a product of half open intervals
    in order to correctly produce a partition) having side length $1/t$ with vertices
    at the appropriate lattice points granting a correct partitioning.
    Let $\{C_i\}_{i=1}^k$ index this family of subcubes, and let $p_i$ be some point
    within each subcube.  Define an approximant
    \[
        f(x) := \sum_{i=1}^k h(p_i) \1(x\in C_i).
    \]
    It follows that, for a point $x\in C_i$ and any $y\in\{-1,+1\}$,
    \[
        |\phi(-y f(x)) - \phi(-y h(x))|
        = |\phi(-y h(p_i)) - \phi(-y h(x))|
        < \tau
    \]
    by construction.  Since $C_i$ was arbitrary, this holds for every subcube; and it
    furthermore holds outside the support of $f$, where $h$ and $f$ are both
    guaranteed to be the constant $0$.

    Combining the various approximation components, it follows that
    \begin{align*}
        \cR(f)
        &= \int \phi(-y f(x)) d\mu(x,y)
        \\
        &\leq \tau \mu(\cX\times\cY) + \int \phi(-y h(x)) d\mu(x,y)
        \\
        &\leq \epsilon/8 + \int \phi(-y \hat g(x))\1(\hat g(x) = h(x)) d\mu(x,y)
        + \int \phi(-y h(x))\1(\hat g(x) \neq h(x)) d\mu(x,y)
        \\
        &\leq \epsilon/8 + \cR_\phi(\hat g)
        + \mu_\cX([\hat g \neq h]) \phi(\|\hat g\|_u)
        \\
        &< \epsilon + \cR_\phi(\fF).
    \end{align*}
    To finish, note by construction that $f$, which was formed from
    axis-aligned subcubes at lattice points within $[-t,t)$,
    satisfies $f \in \SPAN(\cH_t)$ (the indicator for
    each subcube can be modeled as an element of $\cH_t$).
\end{proof}

\begin{proof}[Proof of \Cref{fact:consistency:lsrm}]
    Proceed as in the proof of \Cref{fact:consistency:cheating}, with one modification.
    First determine $\epsilon_i$.
    At each stage, choose $j_i$ large enough so that $\cH_{j_i}$ satisfies
    $\cR_\phi(\SPAN(\cH_{j_i})) < \cR_\phi(\fF) + \epsilon_i$;
    the existence of such a $j_i$
    is straight from the definition of L-SRM families.
    Now choose $m_i$ large enough to satisfy the necessary conditions in the
    proof of \Cref{fact:consistency:cheating};
    meaning the relevant bound from \Cref{fact:deviations} may be instantiated,
    and furthermore these bounds approach zero as $i\to\infty$.
    Now that $m_i$ may be quite massive, as it must now smash the term $n = |\cH_{j_i}|$.
    The proof is otherwise identical to before.
\end{proof}

\end{document}